%% file: paper.tex
\documentclass{article}

\usepackage{paper}
\usepackage[preprint,nonatbib]{neurips_2020}

\usepackage{microtype}
\usepackage{graphicx}
\usepackage{subfigure}
\usepackage{fullpage}
\usepackage{booktabs} % for professional tables

\usepackage{hyperref}

\usepackage{amsmath,amsfonts,amssymb,amsthm}
\usepackage{thmtools,thm-restate}
\usepackage{enumerate}

\newcommand{\etr}{\mathcal{E}_{\mathrm{tr}}} % training environments
\newcommand{\ete}{e_0} % test environment
\newcommand{\rec}{\mathcal{G}} % representation class
\newcommand{\hyc}{\mathcal{F}} % predictor hypothesis class

\usepackage{xcolor}

\begin{document}

\title{Representation Bayesian Risk Decompositions and \\
  Multi-Source Domain Adaptation}
\author[2]{Xi Wu}
\author[1]{Yang Guo}
\author[1]{Jiefeng Chen}
\author[1]{Yingyu Liang}
\author[1,3]{Somesh Jha}
\author[3]{Prasad Chalasani}
\affil[1]{University of Wisconsin-Madison}
\affil[2]{Google}
\affil[3]{XaiPient}
\date{}
\maketitle

\input{abstract}

\section{Introduction}
\label{sec:intro}
\input{intro}

\section{Preliminaries}
\label{sec:pre}
\input{preliminaries}

\section{Single-Source Domain Adaptation}
\label{sec:single-source}
\input{single-source}

\section{Multi-Source Domain Adaptation}
\label{sec:multi-source}
\input{multi-source}

\section{Experiments}
\label{sec:experiments}
\input{experiments}

%\section{Conclusions}
%\label{sec:conclusions}
%\input{conclusions}

\newpage
\section{Broader Impact}
\label{sec:broader-impact}
\input{broader-impact}
\bibliographystyle{abbrv}
\bibliography{paper}
\appendix
\newpage
\input{appendix}

\end{document}

%% file: abstract.tex
\begin{abstract}
  We consider representation learning (hypothesis class $\mathcal{H} =
  \mathcal{F}\circ\mathcal{G}$) where training and test distributions can be
  different. Recent studies provide hints and failure examples for domain invariant
  representation learning, a common approach for this problem,
  but the explanations provided are somewhat different and do not provide
  a unified picture. In this paper, we provide new decompositions of risk which
  give finer-grained explanations and clarify potential generalization issues.
  For Single-Source Domain Adaptation, we give an exact decomposition
  (an equality) of the target risk, via a natural hybrid argument,
  as sum of three factors: (1) source risk, (2) representation conditional
  label divergence, and (3) representation covariate shift.
  We derive a similar decomposition for the Multi-Source case.
  These decompositions reveal factors (2) and (3) as the precise reasons
  for failure to generalize. For example, we demonstrate that domain adversarial
  neural networks (DANN) attempt to regularize for (3) but miss (2),
  while a recent technique Invariant Risk Minimization (IRM) attempts to
  account for (2) but does not consider (3).
  We also verify our observations experimentally.
\end{abstract}

%%% Local Variables:
%%% mode: latex
%%% TeX-master: t
%%% End:

%% file: intro.tex
Representation learning has emerged as a promising approach for
machine learning in domain adaptation~\cite{BDBCP06,ganin2016domain}
(for a more recent analysis of this line, see~\cite{JSR19} and
references therein). A common setup is to consider a hypothesis class
$\calH$ that can be decomposed into $\calF\circ\calG$,
where $\calF$ is a class of predictors which map
representations to predictions\footnote{In this work,
  we assume that the predictors output a probability vector over the
  labels, which corresponds to the output of softmax layer in typical classifiers,
  including deep neural networks.},
and $\calG$ is a class of representations which map inputs to representations.
Compared to using a monolithic hypothesis class, using representations
provides a new level of abstraction to study properties of information useful
for adapting to different domains~\cite{bengio2013representation},
including computer vision~\cite{sener2016learning,dosovitskiy2016inverting}
and natural language processing~\cite{devlin2018bert,peters2018deep}.

A theme of representational domain adaptation is to derive a
\emph{risk decomposition} that involves representations,
and use it to guide the search of desired representations.
For example, a popular decomposition in single-source case
is Domain Invariant Representations (DANN~\cite{ganin2016domain}):
\begin{align}
  \label{eq:dann-obj}
  R^t(f\circ\phi) \le R^s(f\circ\phi) + d(\Phi^s, \Phi^t) +
  \lambda^{\star}_{\calH}
\end{align}
which says that target risk is bounded by three factors:
(1) source risk,
(2) distance between set of feature representations $\Phi^s$ and $\Phi^t$, and
(3) a term $\lambda^{\star}_{\calH}$ that solely depends on the overall
hypothesis class $\calH$ (and thus is regarded as unlearnable).

However, recent work~\cite{JSR19,ZCZG19,ABGLP19} has pointed out that
the term $\lambda^{\star}_{\calH}$ hides information about different
choices of representations, and thus may not be informative about the failure
cases of domain invariant representations. These works proposed failure
examples and possible explanations (e.g.,~\cite{JSR19} proposed
an explanation based on support misalignment).
However, to some extent, these explanations are different from each other and 
do not give a unified picture.

In this paper we take a step to bridge this gap. We derive new
risk decompositions that are more fine-grained and can clarify failure
examples as precise terms in the decompositions.
Our key idea is that since representation class $\calG$ provides
an intermediate abstraction, it is fundamental to understand
the following question: \emph{
  What information does $\phi \in \calG$ elicit for domain adaptation?}

\subsection{Overview of our theory and results}
\label{sec:overview}
As a first step to answer the question, we propose to examine the target risk
where we equip over $\phi$ its Bayesian optimal predictor,
and derive fine-grained risk decompositions. Our risk bounds show that
explicitly incorporating representations can provide novel implications,
and open an avenue for designing future algorithms for representation learning
in domain adaptation. Our results can be broadly categorized into single-source
and multi-source cases.

\noindent\textbf{Single-Source Domain Adaptation (SSDA)}.
We obtain the following results.
\begin{itemize}
\item We derive an \emph{exact decomposition (an equality)}
  of the target risk, based on a natural hybrid argument, as three terms:
  (1) source risk,
  (2) representation conditional label divergence, and
  (3) representation covariate shift.
  We further give an exact decomposition of (3), based on Lebesgue
  decomposition, into (4) representation absolute continuous risk,
  and (5) representation singular risk.

\item This equality allows us to identify a weakness of the invariant
  reprentation approach (DANN) as \emph{mixing the effects of absolute
  continuous risk and singular risk, and may give inferior results
  due to intrinsic representation covariate shift}. It also allows us to
  explain failure examples as found in~\cite{JSR19, ZCZG19} as exactly
  a large conditional divergence (factor (2)), and is
  information-theoretically impossible to solve without labeled data
  from the target distribution. This indicates that
  domain invariant representation approach (e.g. DANN) attempts to regularize (3)
  but misses (2), and there is a \emph{fundamental limitation} of
  Single-Source Domain Adaptation with only \emph{unlabeled} data
  from the target domain.

\item We also analyze the \emph{success} of DANN for
  MNIST$\rightarrow$MNIST-M\footnote{
    Recall that MNIST-M is created by replacing the background of MNIST
    with colored images.} for which, similar to the failure example,
  the input support of two domains is disjoint.
  Our theory again gives an immediate explanation of this success:
  \emph{The perfect representation alignment
    (i.e. factor (3) = $0$) in this case trivially implies
    perfect conditional label alignment
    (i.e. factor (2) = $0$)}.
\end{itemize}

\noindent\textbf{Multi-Source Domain Adaptation (MSDA)}.
We obtain the following results.
\begin{itemize}
\item Multiple training distributions allow us to observe
  conditional label divergence. We derive a risk decomposition that
  target risk is bounded by conditional label divergence and
  covariate shift in the training domains, plus a term called
  predictor adaptation distance quantifying whether these alignments
  in the source domains can generalize to the test domain.

\item Our decomposition reveals that IRM \cite{ABGLP19} considers exactly perfect
  conditional label alignment (factor (2)), but misses representation
  covariate shift (factor (3)), and thus its performance may be hurt
  due to that, which is verified in our experiments
  (Figure~\ref{fig:irm-misalign}). We further note that generalization to
  the target can fail when the predictor adaptation distance is large.
  We demonstrate this via an ``distribution memorization problem''
  (Prop~\ref{prop:distribution-memorization} and
  Section~\ref{sec:dist_mem_milder}).
\end{itemize}
Finally, we perform experiments to confirm our theoretical observations.

%%% Local Variables:
%%% mode: latex
%%% TeX-master: t
%%% End:

%% file: preliminaries.tex
\noindent\textbf{Domain adaptation}. 
Single-source domain adaptation has a source domain $s$ and a target domain $t$.
Each domain is a distribution over a set of feature vectors and labels.
In the multi-source case, we have a set of source domains $\etr$,
and one target domain $\ete$ for testing. Given a representation $\phi$,
we use $\Phi^e$ to denote the random vector $\phi(X^e)$ where $X^e$ is the random variable distributed according to
the input feature distribution in the environment $e$.

\noindent\textbf{Cross entropy function and cross entropy loss}.
For simplicity of developing and presenting results, throughout this
paper we will work with cross entropy loss. However, our results can be
extended in a straightforward way to other loss functions.
Given two distribution $p, q$, cross entropy function $\Ent_p(q)$ is defined as
$\Ent_p(q) = \sum_{i}q_i\log\frac{1}{p_i}$.
We also use cross entropy loss function where for a label $y \in [K]$,
and a probability vector $p \in \Delta_K$,
$\ell(p, y) = \Ent_p({\bf 1}_y) = \log\frac{1}{p_y},$
where ${\bf 1}_y$ is a $K$-dimensional vector with $y$-th component $1$,
and $0$ otherwise. Given an environment $e$ with distribution $X^e, Y^e$,
and a hypothesis $h \in \calF\circ\calG$, we define its population risk
over $e$, $R^e(h)$ as $\Exp[\ell(h(X^e), Y^e)]$.

\noindent\textbf{Representation Bayesian optimal predictors}.
Given $\phi \in \calG$, we denote by $f^e_\phi$ the Bayesian optimal
predictor on top of the representation $\phi(X^e)$ in environment $e$.
That is, $f^e_\phi(\gamma)$ outputs a probability vector such that
for $y \in [K]$, $[f^e_\phi(\gamma)]_y = \Pr[Y^e=y|\; \phi(X^e) = \gamma]$.
In other words, $f^e_\phi(\gamma) = (Y^e|\; \phi(X^e) = \gamma)$,
the label distribution conditioned on $\phi(X^e)=\gamma$.
To simplify notation, we simply use $Y^e|\; \gamma^e$.

%%% Local Variables:
%%% mode: latex
%%% TeX-master: t
%%% End:

%% file: single-source.tex
Motivated by the central question
(``What information does a representation elicit?''),
we propose to examine the risk where we equip over $\phi$
its Bayesian optimal predictor.

\subsection{An Exact Decomposition of Single-Source Representation Risk}
\label{sec:single-source-decomp}
The first step of our single-source decomposition is a hybrid argumment
based on a natural hybrid called $s$-$t$ mixture. This hybrid distribution
retains the same representation distribution as that of the target,
but switches the label distribution conditioned on a representation to that of
the source.
\begin{definition}[\textbf{$s$-$t$ Mixture}]
  An $s$-$t$ mixture, denoted as $(\Phi^m, Y^m)$, is a
  distribution defined on the $\Omega \times [K]$ (representation
  support times label space) as follows:
  (1) $\Phi^m$ and $\Phi^t=\phi(X^t)$ have the same distribution.
  That is the feature distribution follows the target domain.
  (2) On the other hand, $Y^m|\, \gamma^m = Y^s|\, \gamma^s$. That is,
  the conditional label distribution follows the source domain.
\end{definition}

This mixture gives rise to some natural quantities for risk decomposition.
We first consider \emph{representation conditional label divergence}.
Given a representation $\phi$, and a value $\gamma$ that $\phi$ may take,
the conditional label distributions $Y^t|\,\gamma^t$ and
$Y^s|\,\gamma^s$ may differ. We introduce two notions,

\begin{definition}[\textbf{Representation Domain KL-Divergence}]
  \label{def:domain-KL-div}
  We define (representation) domain KL-divergence as
  $\KL^{s,t}_\phi := \int_\Omega
  d_{\KL}\left(f^t_\phi(\gamma)\ \|\ f^s_\phi(\gamma)\right)
  \mu^t(d\gamma)$.
  where $d_{\KL}$ is the KL divergence. Importantly, this quantity
  is natural since it is exactly
  $R^t(f^s_\phi\circ\phi) - R^t(f^t_\phi\circ\phi)$: The gap
  of target risk if we switch predictor from $f^t_\phi$ (target optimal) to
  $f^s_\phi$ (source optimal).
\end{definition}

\begin{definition}[\textbf{Representation Domain Bayesian Divergence}]
  \label{def:domain-Bayesian-div}
  We define (representation) domain Bayesian divergence as
  $\delta^{s,t}_\phi :=
  \int_\Omega \left(\Ent(Y^t|\, \gamma^t) - \Ent(Y^s|\, \gamma^s)\right)
  \mu^t(d\gamma)$. Importantly, this quantity is natural since it is exactly
  $R^t(f^t_\phi\circ \phi) - R^m(f^s_\phi\circ\phi)$: The gap between the risks
  on the target and mixture distribution (recall that target and mixture share
  the same representation distributions; $f^t_\phi$ is optimal for the target,
  and $f^s_\phi$ is optimal for the mixture).
\end{definition}
We refer readers to~\cite{RW15} for a detailed study of the relationship
between the two notions above. Symmetrically, we can consider fixing the
conditional label distributions, but vary the underlying representation
distribution. This gives \emph{representation covariate shift}:

\begin{definition}[\textbf{Representation Covariate Shift}]
  \label{def:repr-covariate-shift}
  We define $s$-$t$ representation covariate shift,
  denoted as $\mu^{s,t}_\phi$, as
  $\mu^{s,t}_\phi = \int_\Omega \Ent(Y^s|\, \gamma^s)\mu^t(d\gamma)
  - \int_\Omega \Ent(Y^s|\, \gamma^s)\mu^s(d\gamma)$.
  In other words, we consider representation distribution changing
  from $\Phi^s$ to $\Phi^t$, while fixing conditional label distribution
  as $Y^s|\, \gamma^s$.
\end{definition}

\begin{lemma}[\textbf{Exact decomposition into conditional divergence
  and covariate shift}]
  \label{lemma:risk-into-cond-div-and-cov-shift}
  We have that
  \begin{align}
    \label{eq:cov-shift-exact-decomposition}
    R^t(f^s_\phi\circ\phi)
    =\underbrace{
    \vphantom{\Big(\Big)} % phantom for alignment
    R^s\left(f^s_\phi\circ\phi\right)}_\text{
    source error}
    + \underbrace{
    \KL_{\phi}^{s,t} + \delta_\phi^{s,t}}_\text{
    conditional label div}
    + \underbrace{
    \mu^{s,t}_\phi}_\text{
    covariate shift}
  \end{align}
\end{lemma}

We next give an \emph{exact} decomposition of the representation covariate shift
$\mu_\phi^{s,t}$. By the Lebesgue decomposition theorem \cite{rudin2006real},
we know that $\mu^t = \mu^t_0 + \mu^t_1$ where $\mu^t_0 \ll \mu^s$ is a measure
that is absolutely continuous in $\mu^s$ and $\mu^t_1$ is a measure
that is singular in $\mu^s$. This decomposition has a natural interpretation
in view of domain adaptation: $\mu^t_0$ represents the target representations
that can be observed in $\mu^s$, wheras $\mu^t_1$ represents the target
representations that \emph{cannot} be observed via $\mu^s$.
For $\mu^t_0$, by the Radon-Nykodym theorem, we have then a function
$\omega_\phi(\cdot)\equiv\frac{d\mu^t_0}{d\mu^s}: \Real^k \mapsto \Real$,
so that for any measurable set $B$:
$\mu^t_0(B) = \int_B \omega_\phi(\gamma)d\mu^s(\gamma)$.
We thus introduce two notions.
\begin{definition}[\textbf{Representation Singular Risk}]
  Let $\tau^{s,t}_\phi \equiv \tau^{s,t}_\phi(\mu_1^t)
  \equiv \int_\Omega \Ent(Y^s|\, \gamma^s)\mu^t_1(d\gamma)$.
\end{definition}
\begin{definition}[\textbf{Representation Absolute Continuous Risk}]
  Let
  $$\zeta^{s,t}_\phi \equiv \zeta^{s,t}_\phi(\mu_0^t)
  \equiv \int_\Omega \left(\omega_\phi(\gamma) - 1\right)
  \Ent(Y^s|\, \gamma^s)\mu^s(d\gamma)$$
\end{definition}

\begin{lemma}[\textbf{Exact decomposition of representation covariate shift}]
  \label{lemma:cov-shift-via-lebesgue-decomposition}
  $\mu^{s,t}_\phi = \zeta^{s,t}_\phi + \tau^{s,t}_\phi$.
\end{lemma}

Combining the above two lemmas we thus arrive at the main theorem for the
single-source case:
\begin{theorem}[\textbf{Exact Decomposition of Single-Source Risk}]
  \label{thm:exact-decomposition}
  We have that
  \begin{align}
    \label{eq:exact-decomposition}
    R^t(f^s_\phi\circ\phi)
    = \underbrace{
    \vphantom{\Big(\Big)} % phantom for alignment
    R^s\left(f^s_\phi\circ\phi\right)}_\text{
    source error}
    + \underbrace{
    \KL_{\phi}^{s,t} + \delta_\phi^{s,t}}_\text{
    conditional label div}
    + \underbrace{
    \zeta^{s,t}_\phi}_\text{
    absolute continuous risk}
    + \underbrace{
    \tau^{s,t}_\phi}_\text{
    singular risk} 
  \end{align}
\end{theorem}

\subsection{Comparison with existing risk decompositions}
\label{sec:comparison}

\noindent\textbf{DANN and intrinsic representation covariate shift}.
One can contrast DANN decomposition (\ref{eq:dann-obj}) with our fine-grained
decomposition, in particular (\ref{eq:cov-shift-exact-decomposition}).
One can see that for common distribution distance function $d(\cdot, \cdot)$
(e.g., MMD), $d(\Phi^s,\Phi^t)$ mixes the effect of absolute continuous risk
and singular risk. More precisely, even if the singular part becomes zero for
a ``right'' representation, there might be nontrivial absolute continuous risk
because there is \emph{intrinsic} covariate shift from $\mu^s$ to $\mu^t_0$.
In this situation, \emph{even if we discover the right representation $\phi$,
$d(\Phi^s,\Phi^t)$ may still be significant and DANN may excessively
modify $\phi$ in order to reduce $d(\Phi^s,\Phi^t)$, leading to
adverse results.}

In fact, some recent proposals (for example, \cite{TLGLLZT18}) made similar
observations, and they considered modifying (\ref{eq:dann-obj}) to align
the conditional representation distributions, $\Phi^s|Y^s$ and $\Phi^t|Y^t$,
instead of $\Phi^s$ and $\Phi^t$. However, in view of our results,
this is only one form of intrinsic covariate shift, and one can easily modify
the representation distributions to break these variants.

\noindent\textbf{Comparison with other bounds}.
We now consider other representative decompositions, specifically:
(T1) Theorem 1~\cite{ben2010theory}, (T2) Theorem 4.1~\cite{ZCZG19},
and (T3) Theorem 2~\cite{JSR19}. More related work are discussed in
Sections~\ref{app:related} and~\ref{app:alter}.
To begin with, the Bayes classifier and our other notions
(Def~\ref{def:domain-KL-div} to~\ref{def:repr-covariate-shift})
are defined w.r.t. the representation. For both (T1) and (T2),
the notions are w.r.t. the input space (e.g., ``Notations'' and
``Comparison with Theorem 2.1'' in~\cite{ZCZG19}).
Working at representation level allows us to examine different representation
conditional distributions in a hypothesis class of representations.
(T1) and (T2) do not formulate representation class.
Our bound is tighter even if one applies (T1) and (T2) at the
representation level. This is because an equality implies that our terms
must be reflected in any valid upper bound, but still, an equality can provide
more thorough insights. For (T1), we provide a detailed comparison in
Appendix~\ref{app:alter}. The insufficiency of (T1) has also been discussed
in several existing works (including~\cite{ZCZG19,JSR19}).

For (T2), we note two more points: (i) Our decomposition is an
``orthogonal decomposition'' but (T2) is not.
Specifically, our conditional label divergence terms
(Def~\ref{def:domain-KL-div} and ~\ref{def:domain-Bayesian-div}) are not
affected by representation covariate shift since both integrals are only
evaluated over the target representation distribution. By contrast,
while the third term in (T2) is related to conditional label divergence,
it depends on both source and target representation distributions,
and so mixes conditional label divergence and covariate shift.
(ii) While the second term in (T2) can be interpreted as covariate shift over
representations, our term provides a precise characterization of the effect
of absolute continuous and singular risks, unveiling a weakness of DANN.

(T3) is the closest decomposition to ours. However their decomposition
is not exact and indeed upper bounds our absolute continuous risk and
singular risk. This again demonstrates the benefits of our equality decomposition.

\noindent\textbf{Controlling covariate shift via source fairness}.
In Section~\ref{app:fairness} we derive an upper bound of the
representation covariate shift that has algorithmic implications.
In that upper bound we consider a notion called \emph{representation
source fairness}, which encourages to find a representation $\phi$
that has uniform performance across different representations $\gamma$.
The notion only depends on the source domain, and can thus be learned with
labeled source data. We note that this notion generalizes a similar theme
considered in recent work~\cite{duchi2018learning} to the representation level.

\subsection{Analysis of examples of domain invariant representations}
\label{sec:analysis-invariant-repr}

We now use our theory to analyze two examples of Domain Invariant Representations.

\begin{example}[{\bf A failure example from~\cite{JSR19, ZCZG19}}]
  \label{example:jsr}
  Consider input space $\calX = [-1,1] \times [-1,1]$,
  $\calG = \{\phi_1, \phi_2\}$ where $\phi_1(x) = x_1$ and $\phi_2(x) = x_2$,
  and $\calF = \{ {\mathbf 1}_\lambda(\cdot) \}$ (that is we consider
  thresholding functions that ${\bf 1}_\lambda(\alpha) = 1$ if $\alpha > \lambda$,
  and $0$ otherwise. The source domain $s$ puts a uniform distribution
  in the second and fourth quadrants, and has label $1$ in the second quadrant,
  and label $0$ in the fourth quadrant. On the other hand, target distribution
  $t$ puts a uniform distribution in the first and third quadrant, and has label
  $1$ in the first quadrant and label $0$ in the third quadrant
  (See Figure~\ref{fig:eg_misalign}). 
  Clearly, the underlying truth is $\phi_2(x) = x_2$, which perfectly classifies
  both source and target data. However, with only unlabeled data from the target
  domain, using (\ref{eq:dann-obj}) we \emph{cannot} distinguish between
  $\phi_1$ and $\phi_2$: Both of them have zero risk on the source domain,
  and both give perfect alignment between $\Phi^s$ and $\Phi^t$.
  (i.e., both perfectly minimize (\ref{eq:dann-obj})).
\end{example}
\begin{figure*}[htb!]
  \centering
  \begin{minipage}{.32\linewidth}
    \centering
    \includegraphics[width=\linewidth]{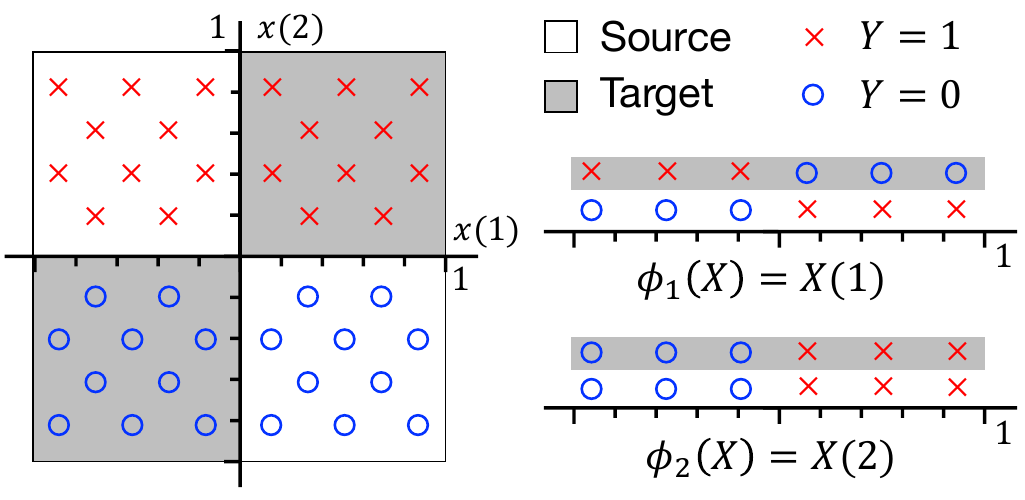}
    \caption{Example from \cite{JSR19} where DANN fails to learn.
      (\ref{eq:dann-obj}) has two different source-optimal solutions
      with different target risks. The figure is from \cite{JSR19}.}
    \label{fig:eg_misalign}
  \end{minipage}
  \qquad\qquad\qquad
  \begin{minipage}{.28\linewidth}
    \includegraphics[width=\linewidth]{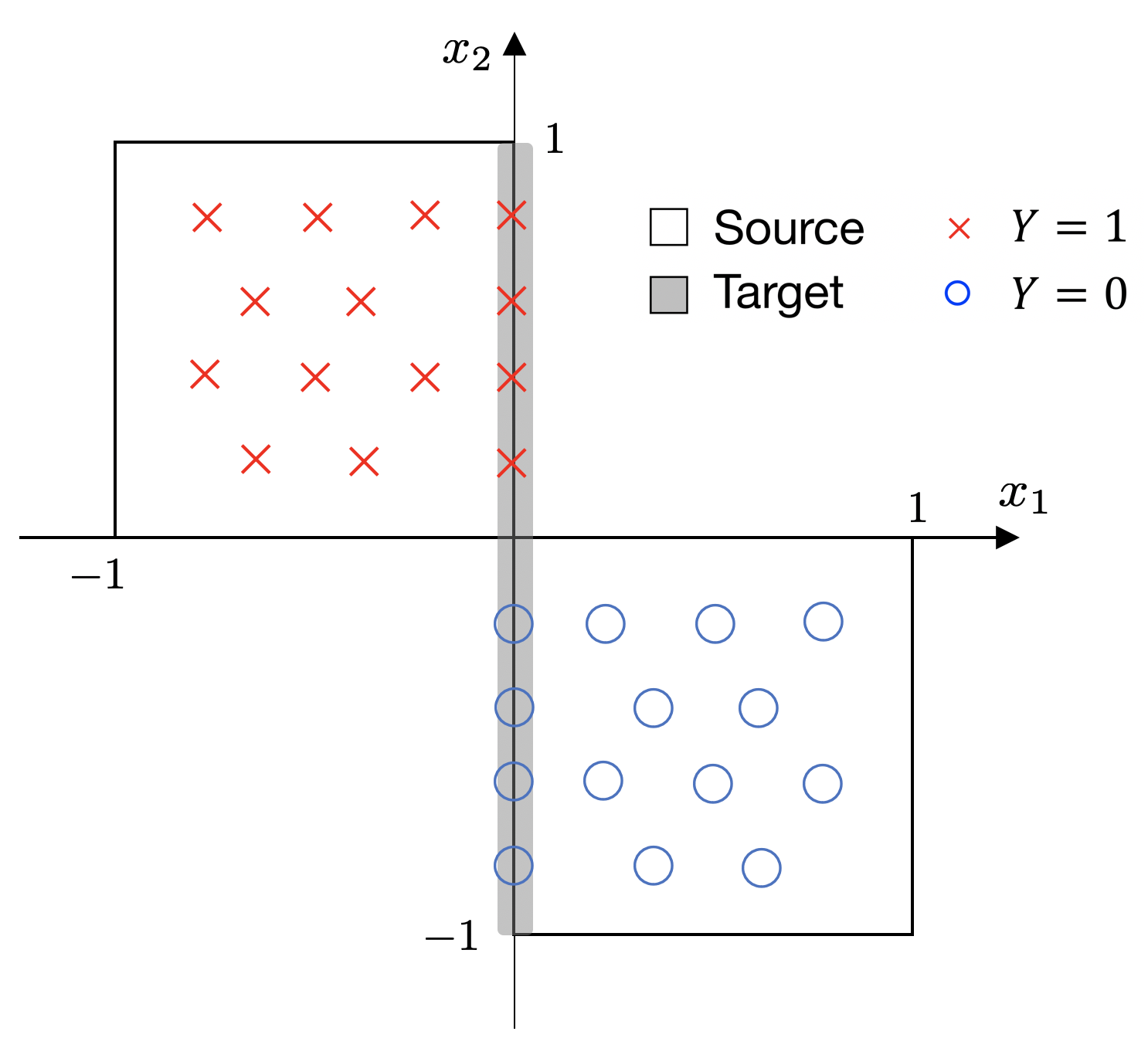}
    \caption[mnist-m]{The source domain is the same,
      but the target domain has $x_1 = 0$.}
    \label{fig:eg_mnistm_abstraction}
  \end{minipage}
\end{figure*}

\noindent\textbf{Our explanation using conditional label divergence}.
Theorem~\ref{thm:exact-decomposition} provides an immediate explanation for
Example~\ref{example:jsr}:
\emph{$\phi_1$ has a large representation conditional label divergence}.
Since we only have one source domain, and do not have labeled data from the
target domain, it is information theoretically impossible to align conditional
label distributions, and thus distinguish between $\phi_1$ and $\phi_2$.
We note that~\cite{ZCZG19} mentioned a similar explanation based
on their Theorem 4.1. As we have discussed in the previous section, our exact
decomposition at representation level provides a more precise explanation
(zero representation covariate shift but large conditional label divergence).

\begin{example}[
  {\bf An example on which DANN succeeds}]
  \label{example:dann-success}
  We consider the same setting as in Example~\ref{example:jsr}.
  However, for target domain, we have uniform distribution over
  $\{0\} \times [-1, 1]$, and for $\{(0, x_2)\ |\ 0 < x_2 < 1\}$ we give label
  $1$, and for $\{(0, x_2)\ |\ -1 < x_2 < 0\}$ we give label $0$. In other words,
  the probability mass, instead of spreading over the second and the fourth
  quadrants, it concentrates on the $x_2$ axis.
  In this case, only $\phi_2(x) = x_2$ aligns the representation distributions,
  since $\phi_1(x) = x_1$ will be constantly $0$ for the unlabeled data
  from the target domain, which has measure $0$ in the source data when projecting
  to $x_1$. DANN will thus learn $x_2$ which perfectly classifies the target data.
\end{example}
\noindent\textbf{The success of DANN on
  MNIST$\rightarrow$MNIST-M}.
The example above captures the essence of the success of DANN on
MNIST$\rightarrow$MNIST-M: \emph{The representation alignment in this case
  trivially implies conditional label alignment}. Merely replacing background
images will make digit representation the only discriminative signal that
exists in both source and target. Therefore by finding the only representaiton
that could align the two domains, the conditional label alignment is
trivially implied.

%%% Local Variables:
%%% mode: latex
%%% TeX-master: t
%%% End:

%% file: multi-source.tex
We now switch to the setting with multiple sources. Multiple source
domains allow us to observe conditional label divergence among source domains,
which one cannot hope to do with a single source (without labeled target data).
Due to the availability of multiple source domains, we focus on the case
where no data (labeled or unlabeled) from the target domain is available
for training.

\subsection{Multi-Source Representation Risk Decomposition}

We observe that, even with multiple source domains, generalization to
a target domain requires connections between the target and sources.
For this we introduce \emph{Predictor Adaptation Gap}. 
\begin{definition}[\textbf{Predictor adaption gap between two distributions}]
  Define the predictor adaptation gap between two distributions
  $e_1$ and $e_2$ with respect to a representation function
  $\phi$ and $\etr$ as
  $d_\phi(e_1, e_2; \etr) \equiv \sup_{e \in \etr} R^{e_1}(f^e_\phi \circ \phi)
  - R^{e_2}(f^e_\phi \circ \phi).$
  Intuitively, a small gap indicates that a small $R^{e_2}(f^e_\phi \circ \phi)$
  implies small $R^{e_1}(f^e_\phi \circ \phi)$. That is, $f^e_\phi \circ \phi$
  can be used in $e_1$.
\end{definition}
\begin{definition}[\textbf{Predictor adaptation gap between target and sources}]
  Define the predictor adaptation gap between $\ete$ and $\etr$ with respect
  to $\phi$ as: $d_\phi(\ete, \etr) \equiv
  \inf_{e' \in \etr} d_\phi(\ete, e'; \etr).$
  We also define the predictor adaptation gap between $\ete$ and $\etr$
  over the whole class $\rec$ as
  $d_\rec(\ete, \etr) \equiv \sup_{\phi \in \rec} d_\phi(\ete, \etr)$. 
\end{definition} 

\begin{theorem}[\textbf{Multi-Source Risk Decomposition}]
  \label{thm:multi-source}
  For any $\phi$, we have
  \begin{align}
    \label{eq:multi-source-decomposition}
    \sup_{e \in \etr } R^{\ete}(f^e_\phi \circ \phi) \le
    \underbrace{
    \sup_{e\in \etr} R^e(f^e_\phi \circ \phi)}_{
    \text{source error}} +
    \underbrace{
    \sup_{e,e' \in \etr}
    [\delta_\phi^{e, e'} + \KL_{\phi}^{e,e'}  + \mu^{e,e'}_\phi]}_{
    \text{cond. label div. + covariate shift}}
    + \underbrace{
    \vphantom{\sup_{e\in \etr} R^e(f^e_\phi \circ \phi)}
    d_\phi(\ete, \etr).}_{
    \text{predictor adaptation gap}}
  \end{align}
\end{theorem}
\vspace{-.3cm}
% The risk bound has three parts: (1) Source risk.
% The supremum over $e \in \etr$ is necessary for getting a small target risk,
% if the set of possible targets $e_0$ includes $\etr$.
% The other two terms quantify the generalization to targets outside $\etr$. 
% (2) Conditional label divergence and covariate shift of the representations
% among the sources. Assuming that there exists a ground-truth function
% $f^*\circ \phi^*$
% that is simultaneously good for all distributions,
% this part can be used for regularization (see the next subsection).
% (3) The ``generalization gap'' due to the distribution shift
% from the sources and the target. This term is the key factor determining
% generalization to new domains/distributions beyond the sources $\etr$
% considered in the training. If the target $\ete$ does not align with $\etr$,
% then in the worst case we cannot expect the learned hypothesis to generalize to
% $\ete$. On the other hand, when $\ete$ aligns with some source in $\etr$,
% then generalization is possible.

Compared with Theorem~\ref{thm:exact-decomposition},
Theorem~\ref{thm:multi-source} has an additional term of predictor adaptation gap.
This is intentional since the predictor gap is related to the target and thus
cannot be optimized in the setting without target data.
Importantly, this bound shows a trade-off between the generalization gap
and the other two terms: \emph{A larger $\etr$ may lead to a smaller gap
  but larger source risks, larger label divergence and covariate shift
  among the sources, and harder optimization}. Similarly, the bound also
shows \emph{a larger hypothesis class $\rec$ potentially leads to smaller source
  risks but a larger gap}. To see this, suppose the optimization method
successfully finds a $\hat\phi$ with small source risks, and small conditional
label divergence and covariate shift among the sources. Then,
the generalization gap is $d_{\hat\phi}(\ete, \etr)$, which can be as large as
$\sup_{\phi \in \rec} d_\phi(\ete, \etr)$ in the worst case.
% So a larger hypothesis class $\rec$ potentially leads to a larger gap.

% In the single-source case, both the label divergence and the covariate shift are
% related to the target domain so we do not separate the two parts.
% This also means that directly applying Theorem~\ref{thm:exact-decomposition}
% to the multi-source setting leads to a bound using conditional label divergence
% and covariate shift between the target and the sources. This does not provide
% insights about the multi-source setting and explanations for the examples
% in this section. 
%the multi-source bound is a more general form of the single-source bound.

\subsection{Conditional Label Divergence and Invariant Risk Minimization}

% Our theorem leads to the following question: how does one regularize to minimize
% the terms? Here we consider the representation conditional label divergence.
% The extreme is to make the conditional label distribution on the representation
% to be the same across $e$ and $e'$, which will make  the label divergence
% (both terms $\delta_\phi^{e, e'}$ and $\KL_{\phi}^{e,e'}$) zero.

We consider the following notion for regularizing conditional divergence.

\begin{definition}[\textbf{Environment Conditional Invariance}]
  A representation $\phi$ satisfies environment
  conditional invariance (ECI) w.r.t. distribution family $\mathcal{E}$ if
  $\forall e, e' \in \mathcal{E}$,
  $\forall r \in \supp(\phi(X^e)) \cap \supp(\phi(X^{e'}))$,
  $\forall y \in [K]$,
  $\Pr[Y^e =y\ |\ \phi(X^e)=r] = \Pr[Y^{e'}=y\ |\ \phi(X^{e'})=r]$.
\end{definition}
ECI means that the Bayesian optimal prediction function on the representation
(i.e., $\Pr[Y^e|\phi(X^e)]$) is invariant across all the distributions.
%This is a natural assumption on the representation $\phi$ and thus is a
%particularly useful inductive bias for invariant learning. Previous work has
% provided extensive motivation.
This notion is closely related to the notion of invariant prediction
in~\cite{PBM16}, and has been mentioned in recent work
(e.g.,~\cite{pan2010domain}). Furthermore, a recent work~\cite{ABGLP19}
of Invariant Risk Minimization (IRM) has proposed and studied a closely related
notion that representation $\phi$ leads to the existence of a predictor
\emph{simultaneously optimal} for all the domains:
%(called $\phi$ eliciting an invariant predictor).
% is proposed and obtained interesting experimental results: 
\begin{align}
  \label{eq:IRM-obj}
  \begin{split}
    \min_{h \in \hyc, \; \phi \in  \rec }
    & \quad \sum_{e \in  \etr } R^{e}(h \circ \phi), \\
    \textrm{ subject to }
    & \quad 
    h \in \arg\min_{h \in  \hyc} R^{e}(h \circ \phi)
    \quad \textrm{~for~}\forall  e \in  \etr . 
  \end{split}
\end{align}
% This is empirical risk minimization subject to \emph{simultaneous optimality}
% of the predictor for all sources.
ECI and IRM are not equivalent if the loss (e.g., $0$-$1$ loss) does not have
the property that the minimizer is the Bayesian optimal predictor.\footnote{
  See Section~\ref{app:eci_irm} for a detailed discussion.
  Therefore, we use ECI for our analysis, since it is a property of the
  representation itself and does not involve the optimization and thus is
  more convenient for the analysis.
  % On the other hand, simultaneous optimality is convenient for training. 
}
If the loss function satisfies the Bayesian optimality property,
and the hypothesis class $\hyc$ contains the Bayes perdictor of representations,
ECI and IRM are equivalent. In this case, let $ \rec _I$ denote the subset of
hypotheses in $ \rec $ that satisfy ECI. Then IRM is equivalent to minimizing
$\sum_{e \in  \etr } R_{e}(h \circ \phi)$ subject to
$h \in  \hyc, \phi \in  \rec _I$.
By Theorem~\ref{thm:multi-source},
the solution $\hat{h} \circ \hat\phi$ satisfies
\begin{align}
  \label{eq:bound-with-ECI}
  & R^{\ete}(\hat{h}\circ \hat\phi) \le \sup_{e\in \etr}
    R^e(\hat{h} \circ \hat\phi) + \sup_{e,e' \in \etr}  \mu^{e,e'}_{\hat\phi} +
    d_{\rec_I}(\ete, \etr).
\end{align}
Compared to the original bound, ECI enforces perfect conditional label alignment,
and also potentially reduces the generalization gap from
$d_\rec(\ete, \etr)$ to $ d_{\rec_I}(\ete, \etr)$ by pruning away those hypothesis
$\phi$ that do not satisfy ECI on the sources. When the ground-truth indeed
satisfies ECI, this will not hurt the sources risks and thus significantly
decreases the bound on the target risk.  

% On the other hand, the last two terms in the bound imply two additional
% conditions for generalization. The covariate shift can be computed using the
% data from the sources, and thus can be incorporated into the training.
% The predictor adaptation gap is a condition on the target domain so it cannot be
% computed without target data, while it can be controlled by choosing  proper
% hypothesis classes.  We discuss these two terms in the following two
% subsections, respectively.

\subsection{Predictor Adaptation Gap} 
% When we find a model with low source risks and small conditional label divergence
% and covariate shift among the sources, can we guarantee that it has a low risk in
% the target domain? Our bound shows that this is not true when the prediction
% adaptation gap $d_{\rec_I}(\ete, \etr)$ is large.
% % In particular, this can happen when the hypothesis classes $ \rec $ are too large. 
% When will the gap be large and lead to a high target risk?
% First, ECI is only imposed on the sources during training while it may not be
% satisfied in the target (i.e., a large conditional label divergence between the
% sources and the target), which can lead to a large gap. Second, even if all
% optimal solutions satisfy ECI in the target, some of them can have large
% representation covariate shift between the sources and the target, and thus still
% have a large target risk.
% (We emphasize that these are different from the
% conditional label divergence and the representation covariate shift among the
% sources, which can be controlled by the training algorithm.)
In this section we study the problem of \emph{distribution memorization}
that may lead to a large predictor adaption gap.
Distribution memorization is similar to overfitting via memorizing
training samples in the traditional supervised learning setting,
but it memorizes the entire distributions rather than the training samples. 
Even if infinite data from each source is available and the hypothesis classes
are just slightly larger than necessary, distribution memorization can happen.
To illustrate this, we consider the following example:
% We now give an intuitive example of distribution memorization caused by
% representation covariate shift between sources and target
% (note that this is different from covariate shift among
% sources considered in Theorem~\ref{thm:multi-source})
% and provide more in Section~\ref{sec:dist_mem_milder}.
Consider the case with classification error, i.e., the label is in
$\{-1, +1\}$ and the loss of $f$ on data $(x,y)$ is
$\ell(f(x),y) = |\mathrm{sign}(f(x)) - y|$. Suppose the support of the target
$\supp(X^{e_0})$ can be disjoint from those of the sources
$\cup_{e \in \etr} \supp(X^{e})$. Assume:
(1) There are ground-truth $\phi^* \in \rec$ and $f^* \in \hyc$, such that
$f^* \circ \phi^*$ has 0 error in all domains (including all sources and also
the target), $\phi^*$ satisfies ECI in all domains, and the distributions of
$\phi^*(X^e)$ are the same for all sources $e$.
(2) The optimization finds $f$ and $\phi$ such that in all sources,
$f \circ \phi$ has 0 error, $\phi$ satisfies ECI, and the distributions of
$\phi(X^e)$ are the same.

\begin{proposition}[\textbf{Distribution Memorization}]
  \label{prop:distribution-memorization}
  There exists an instance of the data distributions and $\rec \circ \hyc$
  satisfying the above assumptions, where there is an optimal solution
  $f\circ \phi$ that satisfies ECI and has 0 risks in all the source domains,
  but in the target domain has a risk $1/2$  which is as large as random guessing.
  Furthermore, in the instance, $\phi(x)$ is simply the concatenation of
  $\phi^*(x)$ with one additional bit, and $f$ is linear. 
\end{proposition}

Intuitively, the representation remembers whether the data is from the target
and then the predictor uses this to make different predictions for the
target domain. More generally, we do not need the support of the target domain to
be disjoint from those of the source domains. A similar phenomena  can happen when
the target has large total variation distances with the sources and the hypothesis
classes are too large.\footnote{
  Section~\ref{sec:dist_mem_milder} provides a more complex example
  where the supports of the target and the sources overlap but
  a large representation covariate shift leads to a large gap.
  It also provides another example where the supports overlap while
  a large conditional label divergence leads to a large gap.}
Our analysis shows that the representation class should be carefully chosen
to alleviate the prediction adaptation gap and consequently get better
generalization to the target domain. The connection between the prediction
adaptation gap and the label divergence and covariate shift (between target
and sources) also suggests that if some (unlabeled) data from the target domain
are available, such data can potentially be used to regularize the gap
explicitly during the training.

% The discussions thus suggest that with large hypothesis classes and
% a large difference between the sources and the target, we can have a large
% conditional label divergence and/or a large representation covariate shift
% between the sources and the target, which is reflected as a large prediction
% adaptation gap term in our bound.

%%% Local Variables:
%%% mode: latex
%%% TeX-master: t
%%% End:

%% file: experiments.tex
In this section we perform experiments to verify our theoretical observations.

\noindent\textbf{SSDA: Representation covariate shift}.
We demonstrate two points:
{\bf (1)} Without considering representation covariate shift, DANN performance
will deteriorate with more significant covariate shift.
{\bf (2)} More importantly, we demonstrate a novel point inspired by our theory
that, \emph{if we ``reweigh'' the points according to the covaraite shift
(i.e., we have an oracle which tells us the representation covariate shift
for the right representation), then DANN works again.}

To do so, we follow the MNIST $\rightarrow$ MNIST-M domain adaptation scenario
from~\cite{ganin2016domain}. To induce representation covariate shift,
the data in the target domain are skewly sampled for each class
according to a weight vector $w$. $w$ is set as follows:
(1) Mild covariate shift case: $\omega[i] = 0.25$ if $i=0$,
$\omega[i]=9$ if $i=9$, and otherwise $\omega[i] \sim {\rm Uniform}([.25, .75])$.
(2) Strong covariate shift case: $\omega[i] = 0.0625$ if $i = 0$,
$\omega[i] = 0.9375]$ if $i=9$, and otherwise
$\omega[i] \sim {\rm Uniform}([.0625, .9375])$.

\begin{figure}[h!]
  \centering
  \includegraphics[width=0.4\linewidth]{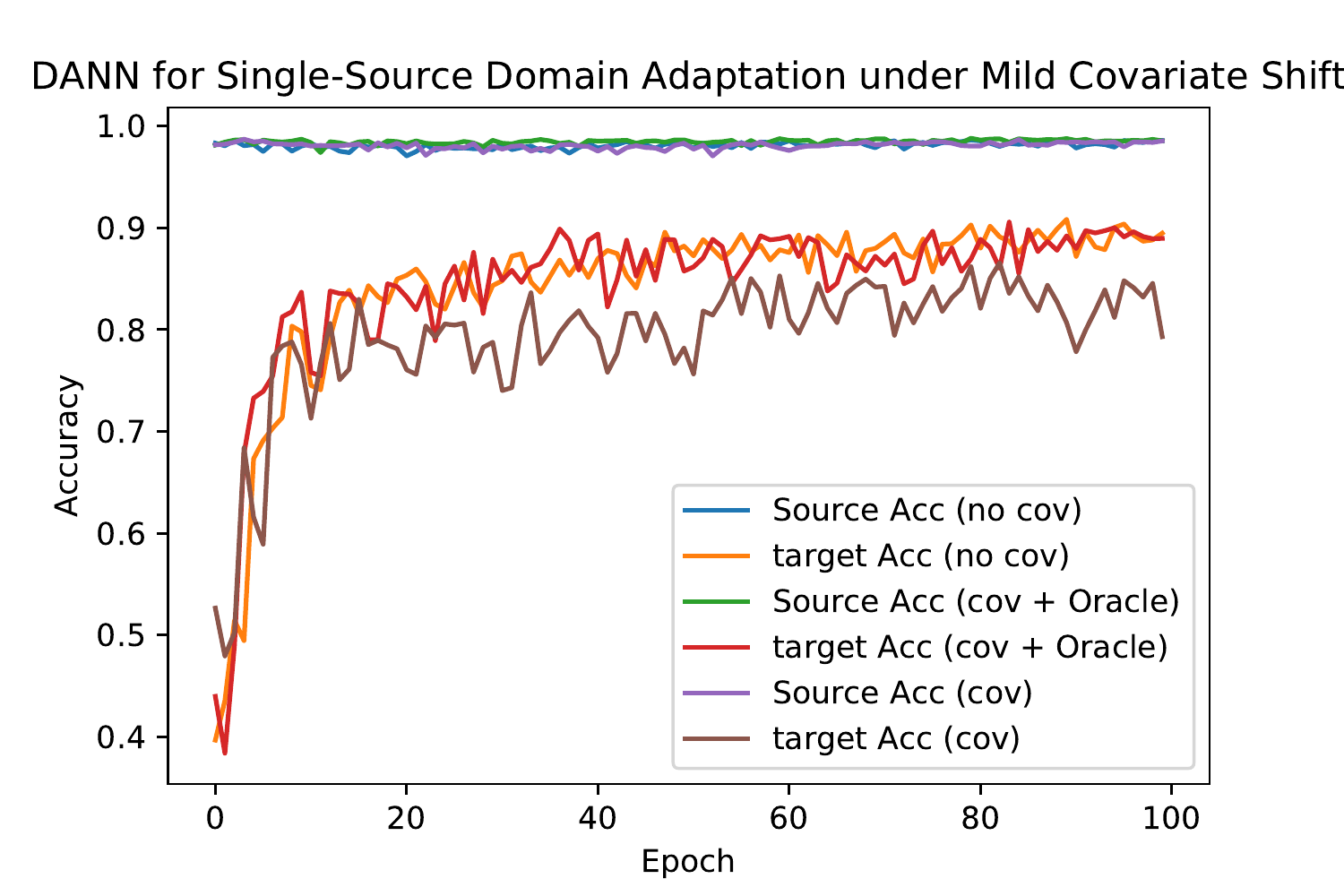}
  \includegraphics[width=0.4\linewidth]{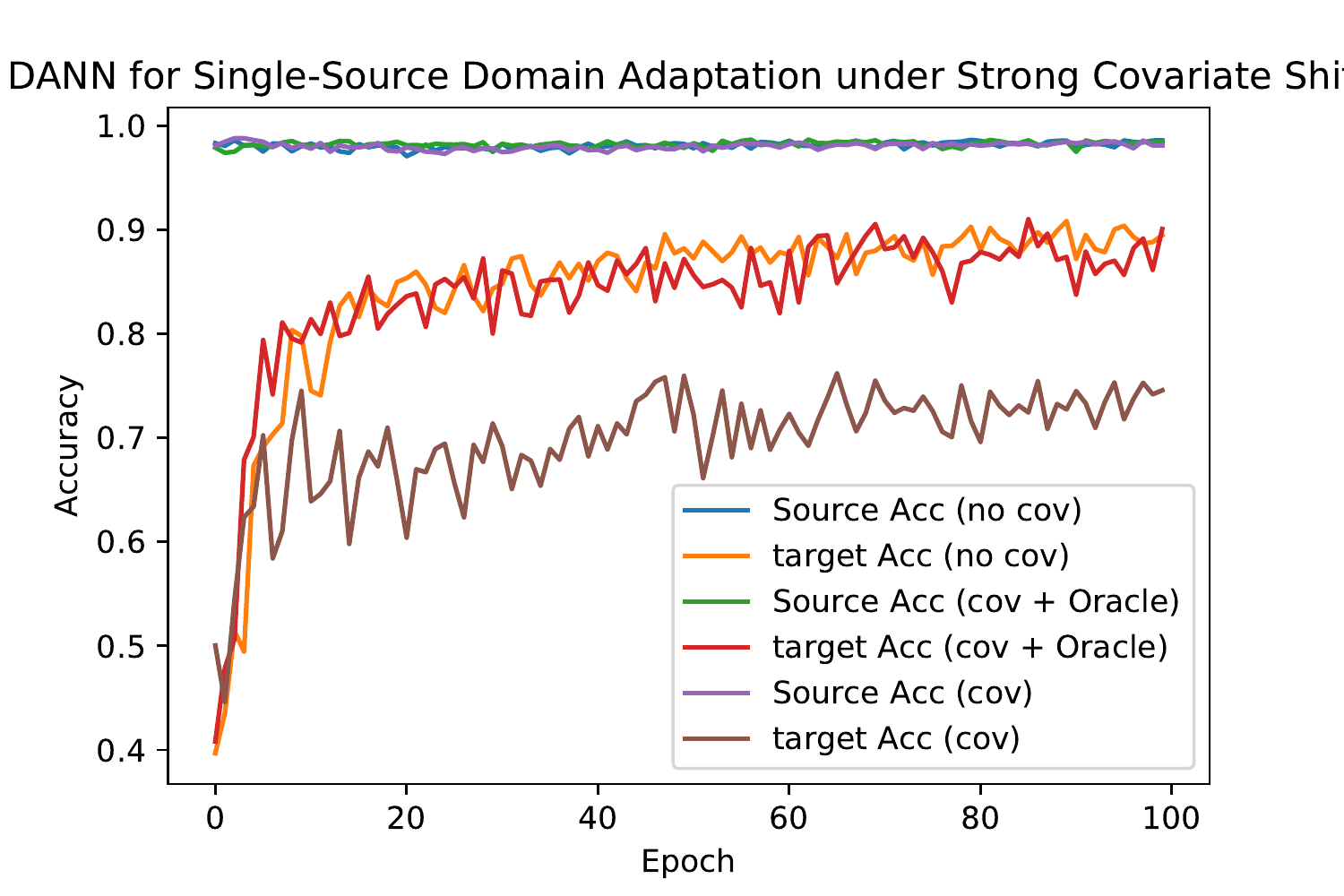}
  \caption{
    Source and Target accuracy for DANN under mild (Left) and strong covariate
    shift (right). In each scenario, we compare three cases:
    no covariate shift (baseline), covariate shift with naive DANN,
    covariate shift with DANN under oracle sampling for the source domain}
  \label{fig:dann-cov-shift}
\end{figure}

Figure~\ref{fig:dann-cov-shift} confirms the gap in the target accuracy between
naive application of DANN and DANN with oracle source sampling is significant:
It increases with the effect of representation covariate, which is measured by
the maximum relative weight ratio in our case. This gap confirms our theoretical
observation that \emph{DANN objectives mix the effect of absolute continous and
  singular risks, which can result in inferior performance.}
This also suggests that the design of domain adaptation algorithms may need to
consider separating the effect of absolute contious and singular risks.

\begin{figure*}[htb]
  \centering
  \begin{minipage}{.3\linewidth}
    \centering
    \includegraphics[width=\linewidth]{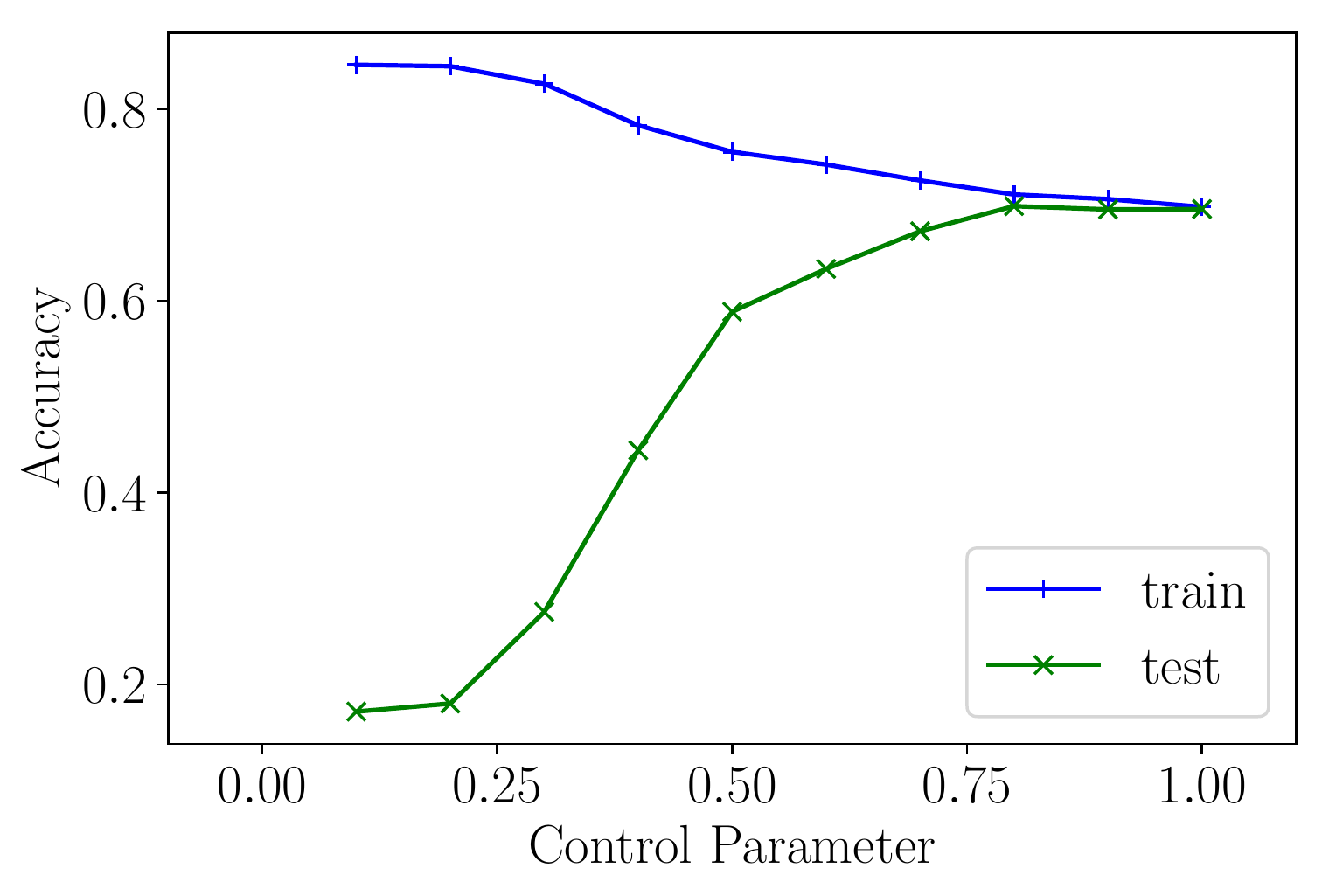}
    \caption{IRM under representation covariate shift.
      Smaller control parameter gives larger representation covariate shift,
      and worse test accuracy.}
    \label{fig:irm-misalign}
  \end{minipage}\quad
  \begin{minipage}{.67\linewidth}
    \centering
    \includegraphics[width=.47\linewidth]{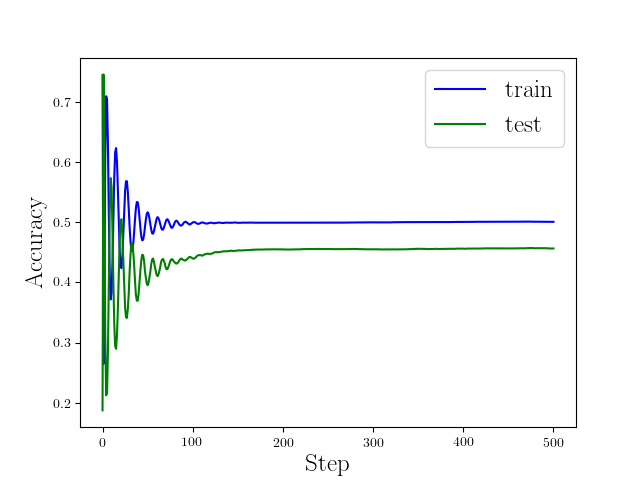}
    \includegraphics[width=.47\linewidth]{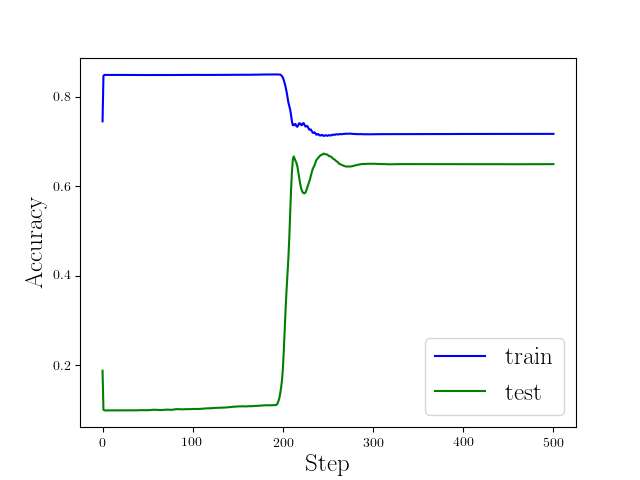}
    \caption{Training and test accuracy of IRMv1v.s.\ epochs.
      Left: one-stage, use the regularization to impose ECI
      for the whole training. Right: two-stage, first train
      without the regularization for 190 steps and then use regularization.
      One can see that the two stage training significantly improves the
      test accuracy.}
    \label{fig:irm}
  \end{minipage}
\end{figure*}
\noindent\textbf{MSDA: IRM and source-target representation covariate shift}.
Our analysis indicates \emph{a large representation distribution shift
can lead to larger target risks, and only enforcing ECI will not suffice.}
Here we provide supporting empirical evidence, by experimenting on a variant of
the colored MNIST dataset from~\cite{ABGLP19}.
In the original construction, we have equal mass on the digits in
both source domains, so there is no representation covariate shift.
We modify the construction process so that the two source domains have
\emph{misaligned}  distributions over the digits: $e_1$ has mass $\frac{p}{1+p}$
on digits 0-4 and $\frac{1}{1+p}$ on digits 5-9, while $e_2$ has mass
$\frac{1}{1+p}$ on digits 0-4 and $\frac{p}{1+p}$ on digits 5-9.  So the shift is
controlled by a single control parameter $p$, as $p$ increases the shift becomes
larger.  Figure~\ref{fig:irm-misalign} shows the results where $p$
increases the test accuracy continues to decrease.\footnote{
  The exact data generating process and results are provided in
  Appendix~\ref{sec:irm-misaligned-data} and Table~\ref{tbl:irm-misalign}.} 
The result confirms our observation that as the representation covariate shift
becomes more significant, models learned on the source domains have worse
generalization to the test domain.

\noindent\textbf{MSDA: IRM, hypothesis class size, and predictor adaptation gap}.
\cite{ABGLP19} proposed an algorithm called IRMv1 for IRM.
We observe that, IRMv1 \emph{fails} to generalize on Color-MNIST when imposing
the ECI regularization for the \emph{whole} training process. On the other hand,
a two-stage training succeeds: First we train without regularization, and then
train with the regularization. Figure~\ref{fig:irm} gives the learning curves for
these two training methods. For this interesting observation,
our multi-source theory provides an explanation that, essentially,
\emph{the first stage is a pretraining which gives a smaller hypothesis class
  that may have smaller predictor adaption gap}. More precisely,
the first stage begins with an initialization $\phi_0$ and finds an intermediate
solution $\phi_1$, and the second stage uses $\phi_1$ as a warm start and searches
in a neighborhood $\mathcal{N}(\phi_1)$ of $\phi_1$ to obtain the final solution
$\phi_2$. Here, $\mathcal{N}(\phi_1)$ can be much smaller than the original
hypothesis class $\mathcal{G}_I$. Then the predictor adaptation gap
reduces from $d_{\rec_I}(\ete, \etr)$ to $d_{\mathcal{N}(\phi_1)}(\ete, \etr)$,
and thus improves generalization. We confirmed this explanation empirically.
We computed the $\ell_2$ distance between the parameters of $\phi_0$ and
$\phi_1$, and for $\phi_1$ and $\phi_2$. The latter is less than $8\%$ of the
former, suggesting that it is indeed doing pre-training and supporting
our explanation. This also suggests that \emph{the two-stage training heuristic
can be a general strategy to improve generalization in domain adaptation.}

%%% Local Variables:
%%% mode: latex
%%% TeX-master: t
%%% End:

%% file: broader-impact.tex
This paper is purely theoretical and has no immediate societal impact.
It may lead to the development of better domain adaptation algorithms,
which may have practical impact.

%%% Local Variables:
%%% mode: latex
%%% TeX-master: t
%%% End:

%% file: appendix.tex
\section{Proofs for Section~\ref{sec:single-source}}
\label{sec:proofs}

\subsection{Proof of Lemma~\ref{lemma:risk-into-cond-div-and-cov-shift}}
We decompose $R^t(f^s_\phi\circ\phi) - R^s(f^s_\phi\circ\phi)$ as
\begin{align*}
  &R^t(f^s_\phi\circ\phi) - R^s(f^s_\phi\circ\phi)
    = \left(R^t(f^s_\phi\circ\phi) - R^t(f^t_\phi\circ\phi)\right) \\
  &\qquad+ \left(R^t(f^t_\phi\circ \phi) - R^m(f^s_\phi\circ\phi)\right) \\
  &\qquad+ \left(R^m(f^s_\phi\circ\phi) - R^s(f^s_\phi\circ\phi)\right).
\end{align*}
One can then verify that
$R^t(f^s_\phi\circ\phi) - R^t(f^t_\phi\circ\phi) = \KL_{\phi}^{s,t}$,
$R^t(f^t_\phi\circ \phi) - R^m(f^s_\phi\circ\phi) = \delta_\phi^{s,t}$,
$R^m(f^s_\phi\circ\phi) - R^s(f^s_\phi\circ\phi) = \mu_\phi^{s,t}$.

\subsection{Proof of Lemma~\ref{lemma:cov-shift-via-lebesgue-decomposition}}
Note that $\mu^{s,t}_\phi =
\int_\Omega \Ent(Y^s|\, \gamma^s)\mu^t(d\gamma)
- \int_\Omega \Ent(Y^s|\, \gamma^s)\mu^s(d\gamma)
= \int_\Omega \Ent(Y^s|\, \gamma^s)\mu^t_0(d\gamma)
- \int_\Omega \Ent(Y^s|\, \gamma^s)\mu^s(d\gamma)
+ \tau_\phi$. Further,
$\int_\Omega \Ent(Y^s|\, \gamma^s)\mu^t_0(d\gamma)
- \int_\Omega \Ent(Y^s|\, \gamma^s)\mu^s(d\gamma)
= \int_\Omega \left(\omega_\phi(\gamma) - 1\right)
\Ent(Y^s|\, \gamma^s)\mu^s(d\gamma) = \zeta_\phi$

\subsection{Upper Bounding Representation Covariate Shift via Source Fairness}
\label{app:fairness}
\input{upper-bound-cov-shift-via-fairness}

\section{Proofs in Section~\ref{sec:multi-source}}

\subsection{Proof of Theorem~\ref{thm:multi-source}}

For any $e\in \etr$,
\begin{align*}
  & R^{\ete}(f^e_\phi \circ \phi) - R^e(f^e_\phi \circ \phi)
  \\
  = & \inf_{e' \in \etr}[R^{\ete}(f^e_\phi \circ \phi) - R^{e'}(f^e_\phi \circ \phi)
      + R^{e'}(f^e_\phi \circ \phi)- R^e(f^e_\phi \circ \phi)]
  \\
  \le & \inf_{e' \in \etr} [R^{\ete}(f^e_\phi \circ \phi) - R^{e'}(f^e_\phi \circ \phi)]
        + \sup_{e' \in \etr}[ R^{e'}(f^e_\phi \circ \phi)- R^e(f^e_\phi \circ \phi)].
\end{align*}

Therefore, taking $\sup_{e \in \etr} $ on both sides and applying the max-min inequality leads to
\begin{align*}
  \sup_{e \in \etr} R^{\ete}(f^e_\phi \circ \phi)
  \le & \sup_{e \in \etr} R^e(f^e_\phi \circ \phi) + d_\phi(\ete, \etr)
        + \sup_{e, e' \in \etr}[ R^{e'}(f^e_\phi \circ \phi)- R^e(f^e_\phi \circ \phi)].
\end{align*}
For the last term, using the same argument as in Theorem~\ref{thm:exact-decomposition},
\begin{align*}
  R^{e'}(f^e_\phi \circ \phi)- R^e(f^e_\phi \circ \phi) = \delta_\phi^{e, e'} + \KL_{\phi}^{e,e'}  + \mu^{e,e'}_\phi.
\end{align*}
This completes the proof.

% \subsection{Proof of Proposition~\ref{prop:distribution-memorization}}
% Suppose $\rec$ is large enough so that we have a $\phi \in \rec$
% satisfying the following: it maps $x$ from any source $e \in \etr$ to
% the concatenation of $\phi^*(x)$ and two bits $v(x) = [1, 0]$;
% it maps $x$ from the target $e_0$ to the concatenation of
% $[\phi^*(x)]$ and $v(x) = [0, 1]$. Suppose $\rec$ is large enough so that
% we have an $h$ with $h(\phi(x))=h^*(\phi^*(x)) + \langle [0, 2], v(x) \rangle$.
% Then for $x$ from any source, $h(\phi(x)) = h^*(\phi^*(x))$, but for $x$
% from the target, $h(\phi(x)) = h^*(\phi^*(x)) + 2$.
% Suppose $h^*(\phi^*(x)) \in [-1,1]$, and the target has an equal mass
% for the two class labels, then $h\circ \phi$ has source risks $0$
% but a large target risk 1/2. Furthermore, it is easy to see that in all sources,
% $\phi^*$ satisfies ECI and the distributions of $\phi^*(X^e)$ are the same.

\subsection{Proof of Proposition~\ref{prop:distribution-memorization}}

Suppose the support of the target $\supp(X^{e_0})$ can be disjoint from those of
the sources $\cup_{e \in \etr} \supp(X^{e})$, and let $v(x) = 0$ if $x$ is from
a source $e \in \etr$ and $v(x)=1$ if $x$ is from the target $\ete$.
Suppose $\rec$ is large enough so that we have a $\phi \in \rec$ that maps
$x$ to the concatenation of $\phi^*(x)$ and $v(x)$. Suppose $f^*$ is linear and
let $\hyc$ be the set of linear functions, then
we have an $f$ with $f(\phi(x))=f^*(\phi^*(x)) + 2 v(x)$.
Then for $x$ from any source, $f(\phi(x)) = f^*(\phi^*(x))$, but for $x$
from the target, $f(\phi(x)) = f^*(\phi^*(x)) + 2$.
Suppose the target has an equal mass for the two class labels,
then $h\circ \phi$ has source risks $0$ but a large target risk 1/2.
Furthermore, it is easy to see that in all sources,
$\phi^*$ satisfies ECI and the distributions of $\phi^*(X^e)$ are the same.

\subsection{Distribution Memorization under Milder Assumptions}
\label{sec:dist_mem_milder}
Proposition~\ref{prop:distribution-memorization} shows large hypothesis classes
can lead to too large predictor adaptation gap, but assuming the support of the
target is disjoint with those of the sources. Here we show that this assumption
is not needed in general, but just for the simplicity of the presentation and
illustration of intuition.

Consider the following example. The input $x$ lies on the real line.
The conditional probability of the label $Y|X$ are the same for all domains: $Y=0$ on $[-2,-1] \cup [1,2]$, $Y=1$ on $[-1, 1]$, and $\Pr[Y=0|x] = \Pr[Y=1|x] = 1/2$ for any $x \in [2,3]$. The distributions of $X$ are specified as follows.
\begin{enumerate}
\item The target domain $e_0$ puts uniformly mass $\epsilon$ on the interval $[-2, 0]$, mass $\epsilon$ on $[0, 2]$, and mass $1-2\epsilon$ on $[2,3]$.
\item Source $e_1$ puts uniformly mass $1-2\epsilon$ on the interval $[-2, 0]$, mass $\epsilon$ on $[0, 2]$, and mass $\epsilon$ on $[2,3]$.
\item Source $e_2$ puts mass $\epsilon$ on the interval $[-2, 0]$, mass $1-2\epsilon$ on $[0, 2]$, and mass $\epsilon$ on $[2,3]$.
\end{enumerate}
Then $\phi(x) = |x|$ and the classifier $f(\phi(x)) = {\mathbf 1}[\phi(x) \le 1]$ have the optimal error and satisfy ECI on the sources, but  still has a large error $(1-2\epsilon)/2$ in the target domain. This is reflected by a large predictor adaptation gap. In this particular example, the gap is due to the covariate shift between the sources and the target (similar to the example in Proposition~\ref{prop:distribution-memorization}).

\begin{figure*}[htb!]
  \centering
  \includegraphics[width=0.3\linewidth]{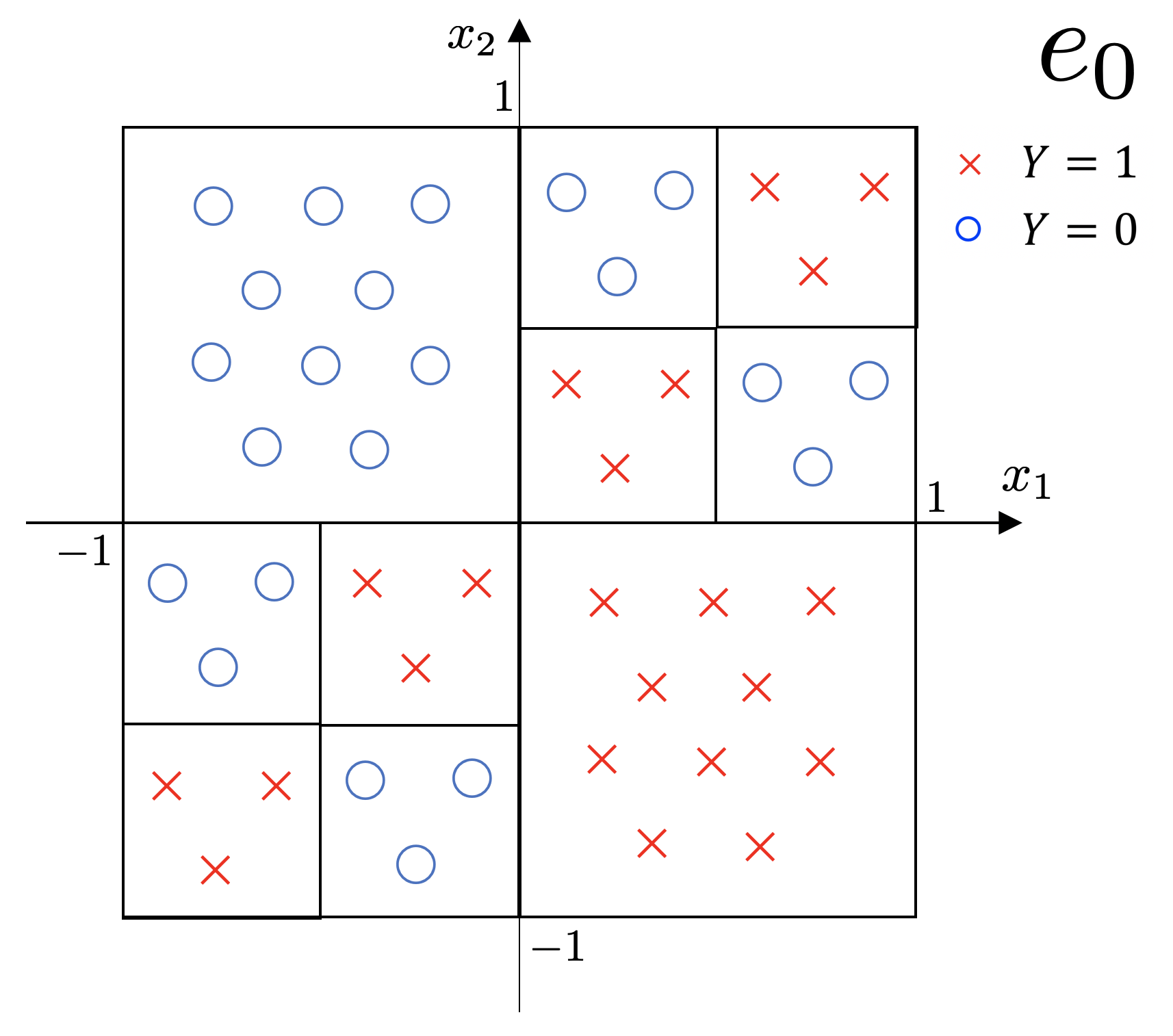}
  \includegraphics[width=0.3\linewidth]{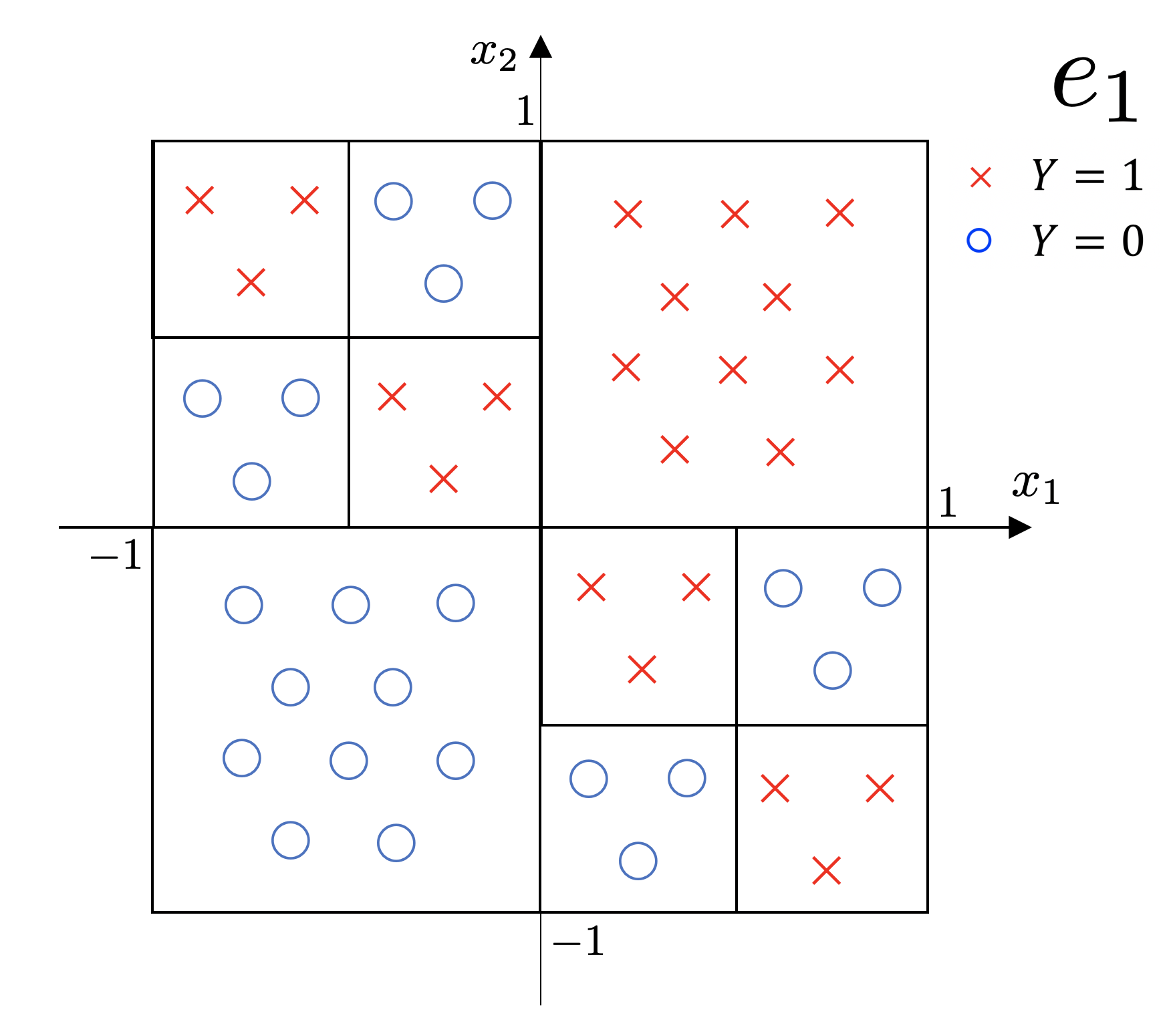}
  \includegraphics[width=0.3\linewidth]{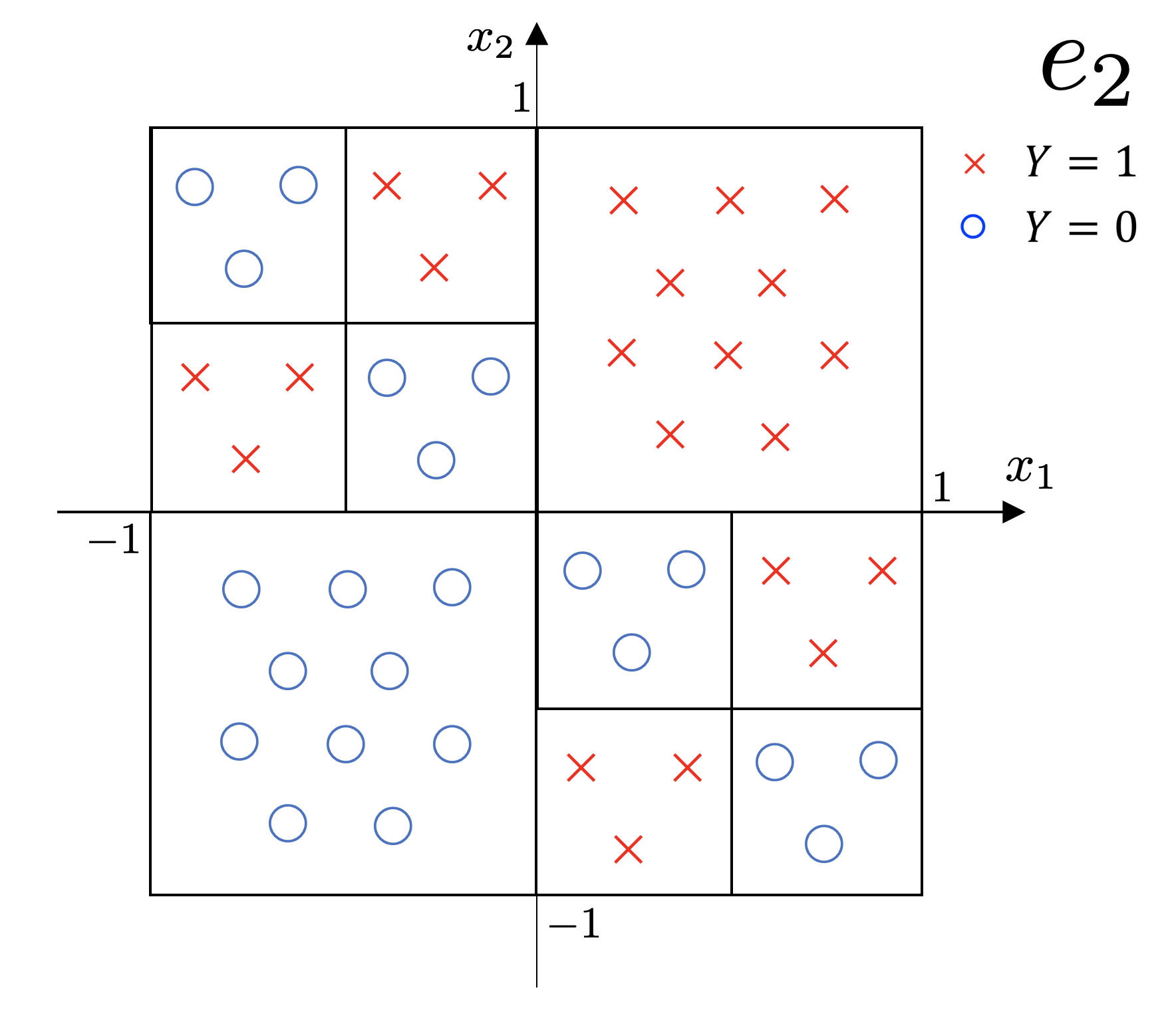}
  \caption{Illustrating example of distribution memorization: $e_1$ and $e_2$ are the two source environments, $e_0$ is the target environment. Both $\phi_1(x) = x_1$ and $\phi_2(x) = x_2$ satisfy the source ECI and zero source covariate shift. However, $\phi_2$ will lead to a large target error. }\label{fig:distribution-memorization}
\end{figure*}

Consider another example, shown in Figure~\ref{fig:distribution-memorization}. It is a variant of Example~\ref{example:jsr}. The input space $\calX = [-1,1] \times [-1,1]$,
$\calG = \{\phi_1, \phi_2\}$ where $\phi_1(x) = x_1$ and $\phi_2(x) = x_2$,
and $\calF = \{ {\mathbf 1}_\lambda(\cdot) \}$ (that is we consider
thresholding functions that ${\bf 1}_\lambda(\alpha) = 1$ if $\alpha > \lambda$,
and $0$ otherwise). The distributions are specified as follows. Let $\epsilon > 0$ be a sufficiently small constant.

\begin{enumerate}
\item
  The target $e_0$ puts uniformly mass $1/2 - \epsilon$ in the second and fourth quadrants,
  and mass $\epsilon$ in the first and third quadrants. It has label $1$ for the fourth quadrant and label $0$ for the second quadrant.
  In the first and third quadrant, it has label $1$ for points in $[-1, -1/2] \times [-1, -1/2]$ or $[-1/2, 0] \times [-1/2, 0]$ or $[0, 1/2] \times [0, 1/2]$ or $[1/2, 1] \times [1/2, 1] $, and has label $0$ for the other points.
\item
  Source $e_1$ puts uniformly mass $1/2 - \epsilon$ in the first and third quadrants,
  and mass $\epsilon$ in the second and fourth quadrants. It has label $1$ for the first quadrant and label $0$ for the third quadrant.
  In the second and fourth quadrant, it has label $1$ for points in $[-1, -1/2] \times [1/2, 1]$ or $[-1/2, 0] \times [0, 1/2]$ or $[0, 1/2] \times [-1/2, 0]$ or $[1/2, 1] \times [-1, -1/2]$, and has label $0$ for the other points.
\item
  Source $e_2$ puts uniformly mass $1/2 - \epsilon$ in the first and third quadrants,
  and mass $\epsilon$ in the second and fourth quadrants. It has label $1$ for the first quadrant and label $0$ for the third quadrant.
  In the second and fourth quadrant, it has label $0$ for points in $[-1, -1/2] \times [1/2, 1]$ or $[-1/2, 0] \times [0, 1/2]$ or $[0, 1/2] \times [-1/2, 0]$ or $[1/2, 1] \times [-1, -1/2]$, and has label $1$ for the other points.
\end{enumerate}
So both $\phi_1$ and $\phi_2$ lead to the optimal error and satisfy ECI in the sources. But $\phi_1$ and the corresponding classifier ${\mathbf 1}_0(\cdot)$ lead to a small error $\epsilon$ in the target, while $\phi_2$ and the corresponding classifier ${\mathbf 1}_0(\cdot)$ lead to a large error $1-\epsilon$ in the target. Again, this is reflected by a large predictor adaptation gap. But in this particular example, the gap is due to the representation conditional label misalignment between the sources and the target.

In summary, both the representation conditional label misalignment and the covariate shift between the sources and the target can lead to a large predictor adaptation gap and consequently a large generalization gap, even when we can make sure the representation conditional label misalignment and the covariate shift among the sources are small.
The precise relationship between the predictor adaptation gap and the misalignment/covariate shift between the sources and the target is left for future work.

\section{Relationship between ECI and IRM} \label{app:eci_irm}

Recall that the IRM approach proposed by~\cite{ABGLP19} is to find
$\hat{h}, \hat{\phi}$ by:
\begin{align}
  \min_{h \in  \hyc, \phi \in  \rec }
  &  \quad \sum_{e \in  \etr } R^{e}(h \circ \phi), \\
  \textrm{ subject to } &  \quad
                          h \in \arg\min_{h \in  \hyc} R^{e}(h \circ \phi)
                          \textrm{ for any }  e \in  \etr .
\end{align}
This is empirical risk minimization subject to \emph{simultaneous optimality} of
the predictor for all sources. As pointed in~\cite{ABGLP19}, when the loss has
the property that the minimizer is the Bayesian optimal predictor and $\hyc$ is
large enough to include that, ECI and simultaneous optimality are equivalent.
Specifically we consider the following definition:

\begin{definition}[{\bf $\phi$-Bayesian Optimality Property}]
  \label{def:bayesian-optimality-property}
  Let $\phi: \calX \mapsto \calR$ be a representation,
  $\ell: \Delta_K \times [K] \mapsto \Real^{+}$ be a loss function,
  where $\Delta_K = \{ (p_1, \dots, p_K)\ |\ p_i \ge 0, \sum_{i=1}^Kp_i=1\}$
  is the $K$-dimensional probability simplex.
  Consider the following optimization problem:
  \begin{align}
    \label{eq:bayesian-optimal-optimization}
    \minimize_{w:\calR\mapsto\Delta_K} \Exp[\ell(f(\phi(X)), Y)]
  \end{align}
  where the expectation is taken over $X, Y$.
  We say that \emph{
    $\ell$ has the Bayesian optimality property with respect to $\phi$},
  if the optimal solution $f^*: \calR \mapsto \Delta_K$
  of (\ref{eq:bayesian-optimal-optimization}),
  which maps a representation to a probability vector, satisfies that
  \begin{align*}
    \forall \gamma \in \supp(\Phi), \forall y \in [K]:
    f^*(\gamma)_y = \Pr[Y=y\ |\ \Phi = \gamma]
  \end{align*}
\end{definition}

Note that the simultaneous optimality is required for some $h\in \hyc$,
while ECI or invariant predictor doesn't require $h$ to be from $\hyc$.
When the loss function has the Bayesian optimality property, $\phi$ satisfying ECI is equivalent to $\phi$ eliciting an invariant predictor (see the discussion later).
We prefer to center our analysis around ECI rather than invariant predictor or simultaneous optimality for convenience, while simultaneous optimality is very useful for enforcing ECI in training.

Here, we analyze IRM under the following assumptions:
\begin{enumerate}
\item[(A1)] The loss has the Bayesian optimality property.
\item[(A2)] $ \hyc$ is sufficiently large to include the conditional probabilities $g(r) = \Pr(Y^e|\phi(X^e)=r)$ for any $\phi \in  \rec $ and any $e \in \etr \cup \{\ete\}$.
\end{enumerate}
Under (A1)(A2), simultaneous optimality is equivalent to $\phi$ satisfying ECI.

It is worth noting many natural loss functions (e.g. squared loss, cross entropy) satisfies Bayesian optimality property. Combining (A1) and (A2), we have the following proposition:
\begin{proposition}
  \label{prop:well-def}
  Let $\phi$ be a representation, $\ell$ be a loss function
  that satisfies the Bayesian optimality property w.r.t. $\phi$,
  and $\calE$ be an environment family. Suppose that $\phi$ is conditionally
  invariant w.r.t. $\calE$. Assuming (A2), then there is a universal optimal solution
  $f_\phi \in \hyc$ to the optimization problem $\min_h \Exp[\ell(h(\phi (X^e), Y^e))]$
  across all $e \in \calE$.
\end{proposition}
\begin{proof}
  Define $f_\phi$ as
  \begin{align*}
    [f_\phi(r)]_y := [f_\phi^e(r)]_y= \Pr[Y^e = y\ |\ \phi(X^e) = r], y \in [K] \quad
    \text{ for any $e \in \calE$ that $\phi \in \supp(\phi(X^e))$}
  \end{align*}
  We note that $f_\phi$ is consistently defined because $\phi$ is conditionally
  invariant w.r.t. $\calE$. Clearly, $f_\phi$ is optimal
  because $\ell$ satisfies Bayesian optimality property.
\end{proof}

Now, given an environment family $\calE$, and (A1) (A2) satisfied, by Proposition~\ref{prop:well-def}, we can consider the following objective:
\begin{align}\tag{\bf ERM-ECI}
  \label{eq:erm-eci}
  \begin{split}
    &\minimize_{\phi} \sum_{e \in \calE} \Exp[\ell(f_\phi(\phi(X^e)), Y^e)] \\
    &\text{subject to } {\tt ECI}(\phi, \calE)
  \end{split}
\end{align}

\begin{proposition}
  (\ref{eq:erm-eci}) is exactly the (\ref{eq:irm}) objective defined as
  \begin{align}\tag{\bf IRM}
    \label{eq:irm}
    \begin{split}
      &\minimize_{h, \phi} \sum_{e \in \calE} \Exp[\ell(h(\phi(X^e)), Y^e)] \\
      &\text{subject to } (\forall e \in \calE)\
      h \in \argmin_{\overline{h}} \Exp[\ell(\overline{h}(\phi(X^e)), Y^e)]
    \end{split}
  \end{align}
\end{proposition}
\begin{proof}
  Because $\ell$ satisfies the conditional expectation property, therefore
  we know that for every $e \in \calE$ the optimal solution will output
  the optimal conditional probability. Therefore for (IRM),
  the only possibility that there is an invariant optimal solution $h$ across
  all environments, is that $\phi$ is conditionally invariant w.r.t. $\calE$.
  However, then we know that the invariant optimal solution $h$ in (IRM) is
  nothing but the $f_\phi$. The proof is complete.
\end{proof}

Without (A1)(A2), simultaneous optimality may not impose ECI; see an example in the next subsection.

\subsection{Example Showing the Difference of ECI and IRM}

\subsubsection{Review of the colored-MNIST Experiment}
In the paper~\cite{ABGLP19}, an interesting experiment on colored-MNIST is performed.
The experiment is essentially as follows:
\begin{enumerate}
\item We start by considering a random variable $G$ which encodes digits.
  Specifically, $G$ is a random variable on $\Real^d$ of pixels.
  We abuse the notation to use $G$ to denote the true digit its pixels encode
  (e.g. $G=0$ means a sample that encodes $0$).

\item We then define a Bernoulli random variable $X$ as
  \begin{align*}
    X = \begin{cases}
      0 & \text{ if $G = 0,1,2,3,4$}, \\
      1 & \text{ if $G = 5,6,7,8,9$}
    \end{cases}
  \end{align*}
  In other words, $X=0$ if the digit encoded in $G$ is less than $5$, and $1$ otherwise.

\item The \emph{true label} $Y$ is generated by flipping $X$
  with probability $.25$.
  That is,
  \begin{align*}
    Y = \begin{cases}
      X & \text{w.p. .75}, \\
      1-X & \text{w.p. .25}
    \end{cases}
  \end{align*}
  In other words, the predictability\footnote{
    We define the \emph{predictability} of a binary random variable $Y$
    using another binary random variable $X$ as $\max\{\Pr[Y=X], \Pr[Y=1-X]\}$.
  } of $Y$ using $X$ is $\Pr[X=Y] = .75$.

\item Then we create a \emph{color random variable} $Z$,
  by flipping $Y$ with probability $q$ (define $p = 1-q$). That is,
  \begin{align*}
    Z = \begin{cases}
      Y & \text{w.p. $p$}, \\
      1-Y & \text{w.p. $q$}
    \end{cases}
  \end{align*}
  That is, the predictability of $Y$ using $Z$ is $p$ if $p > 1/2$,
  and $q$ if $p \le 1/2$.

\item Finally, after the color $Z$ is sampled, we create a new pixel random variable
  $\widetilde{G}$, by coloring the pixels of the digit in $G$ using color $Z$ (red if
  $Z = 0$ and green if $Z=1$). Clearly, the causal structure is
  \begin{align}\tag{{\bf Causal Structure}}
    \label{eq:causal-structure}
    \begin{split}
      &G \xrightarrow{\hspace*{2.6cm}} \widetilde{G} \\
      &\downarrow\qquad\qquad\qquad\qquad\ \ \uparrow \\
      &X \xrightarrow{\hspace*{1cm}} Y \xrightarrow{\hspace*{1cm}} Z
    \end{split}
  \end{align}
  Correlation between $Y$ and $Z$ is variant and thus is spurious.
  Note that both $x$ and $z$ can be recovered from $\widetilde{g}$.

\item The task is to train a classifier to
  \emph{predict $Y$ from $\widetilde{G}$}
  (that is a model $\widetilde{G} \mapsto Y$).
  The experiment in~\cite{ABGLP19} defines three environments:
  \textbf{($\bf e_1$)} where $q = .1$, which generates $Z^{e_1}$. Note that
  $\Pr[Y=Z^{e_1}] = .9 > .75 = \Pr[X=Y]$.
  \textbf{($\bf e_2$)} where $q = .2$, which generates $Z^{e_2}$. Note that
  $\Pr[Y=Z^{e_2}] = .8 > .75 = \Pr[X=Y]$.
  {\bf ($\bf e_3$)}  (test environment): where $q = .9$, which generates $Z^{e_3}$.
  Note that now $\Pr[Y=Z^{e_3}] = .1 \ll .75 = \Pr[X=Y].$ That is, while in training
  environments $Z$ is highly predictive, in the test environment it
  is poorly performing (and instead it is $1-Z$ that is highly predictive).
\end{enumerate}

IRM paper uses $e_1$ and $e_2$ for training. It is straightfoward now to
instantiate both (IRM) and (IRMv1) objectives with the above setting.
Interestingly, with (IRMv1),~\cite{ABGLP19} found
that they can learn to use $X$, but not $Z$.
In a nutshell, they claim that, even with the following two assumptions:
\begin{enumerate}
\item
  The correlation between $Y$ and $Z$ \emph{varies} over training environments.
\item
  In every training environment $Z$ is more predictive than $X$ in predicting $Y$.
\end{enumerate}
IRM can still learn \emph{not} to use correlations that are \emph{not} invariant.

\subsubsection{Example where IRM Does Not Impose ECI}

We now prove that if we use the 0-1 loss (which does not have the Bayesian optimality property), then the optimal solutions to (IRM) in color-MNIST do not satisfy ECI and should learn the spurious correlation $Z$ (i.e., the color).

To start with, we consider $0$-$1$ loss,
that is, given hypothesis $h$ that maps $g \sim \widetilde{G}$ to $\zo$,
\begin{align*}
  \ell(g, y; h) = {\mathbbm 1}[h(g) \neq y] =
  \begin{cases}
    1 & \text{ $h(g) \neq y$,} \\
    0 & \text{ otherwise.}
  \end{cases}
\end{align*}
and therefore $R^e(h)$ is defined to be
$\sum_{(g,y) \sim \widetilde{G}^e}\ell(g, y; h)$.

Our construction has two steps:
First, we construct one optimal solution $(\Phi^*, w^*)$ to (IRM),
but which learns the spurious correlation $Z$.
Second, we prove that \emph{any} optimal solution should learn
the spurious correlation $Z$.
\vskip 5pt
\noindent\textbf{Constructing an optimal $(\Phi^*, w^*)$ to (IRM).}
Now, we construct representation $\Phi^*$ and classifier $w^*$:
\begin{itemize}
\item We let $\Phi^*$ be the representation that maps a colored image
  $g \sim \widetilde{G}$ to a binary vector in $\{0,1\}^2$:
  \begin{align*}
    \Phi^*(\widetilde{G}) =
    \begin{bmatrix}
      X \\
      Z
    \end{bmatrix}
  \end{align*}
  That is, from $\widetilde{G}$, $\Phi^*$ optimally reconstructs the
  digit concept $X$ and color concept $Z$.
\item We construct classifier $w^*$ as
  \begin{align*}
    w^* =
    \begin{bmatrix}
      0 \\
      1
    \end{bmatrix}
  \end{align*}
  In other words, $(w^*)^{\transpose}\Phi^*(\widetilde{G}) = Z$, which
  simply outputs the color concept.
\end{itemize}
We have the following proposition,
\begin{proposition}
  \label{prop:one-optimal}
  For $0$-$1$ loss, $(\Phi^*, w^*)$ is an optimal solution to (IRM).
  Specifically, outputting color $Z$ using $w^*$ is optimal in $e_1$ and $e_2$
  respectively,   and achieves minimal empirical risk combining environments
  $e_1$ and $e_2$.
\end{proposition}
\begin{proof}
  Consider the Bayesian optimal classifier $c^*$ given $X, Z$. That is
  \begin{align*}
    c^*(x, z) =
    \begin{cases}
      1 & \text{if } \Pr[Y=1|x, z] > 1/2 \\
      0 & \text{otherwise.}
    \end{cases}
  \end{align*}
  For any predictor $f: \widetilde{G} \mapsto \zo$, we show that
  $\Pr[Y \neq f(\widetilde{G})] \ge \Pr[Y \neq c^*(X, Z)]$.
  That is $c^*(X, Z)$ achieves the optimal error among all predictors over $\widetilde{G}$.
  To see this, note that from (\ref{eq:causal-structure}),
  we have that $Y \Perp \widetilde{G}\ |\ (X, Z)$.
  Thus $Y \Perp f(\widetilde{G})\ |\ (X, Z)$. Therefore by the law of total expectation
  \begin{align*}
    \Pr[Y \neq f(\widetilde{G})]
    =& \Exp_{X,Z}[\Exp[\mathbbm{1}\{Y \neq f(\widetilde{G})\}\ |\ X, Z]] \\
    =&\sum_{x,z}p(x,z) \cdot \bigg( p(Y=1, f(\widetilde{G})=0\ |\ x,z) + p(Y=0, f(\widetilde{G})=1)\ |\ x,z) \bigg) \\
    =&\sum_{x,z}p(x, z) \cdot \bigg(p(Y=1|x,z)p(f(\widetilde{G})=0|x,z) + p(Y=0|x,z)p(f(\widetilde{G})=1|x,z)) \bigg) \\
    \ge& \sum_{x,z} p(x, z) \cdot \min\bigg\{ p(Y=1|x,z), p(Y=0|x,z) \bigg\} \\
    =& \sum_{x,z} p(x, z) \cdot \Pr[Y \neq c^*(x, z)] \\
    =& \Pr[Y \neq c^*(X, Z)]
  \end{align*}
  Clearly, $\Phi^*(\widetilde{G}) = (X, Z)$. Next we show that $c^* = w^*$.
  For each environment we can compute the Bayesian optimal predictor
  $\Pr[Y=y\ |\ X=x, Z=z]$, for $x,y,z \in \zo$. We have that,
  \renewcommand{\arraystretch}{1.3}
  \begin{center}
    \begin{tabular}{ c|c|c }
      \hline
      \textbf{$\bf e_1$} & $y=0$ & $y=1$ \\
      \hline
      $x=0, z=0$ & $\bf \frac{27}{28}$ & $\frac{1}{28}$ \\
      \hline
      $x=0, z=1$ & $\frac{1}{4}$ & $\bf \frac{3}{4}$ \\
      \hline
      $x=1, z=0$ & $\bf \frac{3}{4}$ & $\frac{1}{4}$ \\
      \hline
      $x=1, z=1$ & $\frac{1}{28}$ & $\bf \frac{27}{28}$ \\
      \hline
    \end{tabular}
    \quad
    \begin{tabular}{ c|c|c }
      \hline
      \textbf{$\bf e_2$} & $y=0$ & $y=1$ \\
      \hline
      $x=0, z=0$ & $\bf \frac{12}{13}$ & $\frac{1}{13}$ \\
      \hline
      $x=0, z=1$ & $\frac{3}{7}$ & $\bf \frac{4}{7}$ \\
      \hline
      $x=1, z=0$ & $\bf \frac{4}{7}$ & $\frac{3}{7}$ \\
      \hline
      $x=1, z=1$ & $\frac{1}{13}$ & $\bf \frac{12}{13}$ \\
      \hline
    \end{tabular}
  \end{center}
  For each row, we highlight (bold) the cell which Bayesian optimal predictor
  should output. One can see that for either environment, the Bayesian
  optimal predictor is simply to output $z$. This shows that:
  \begin{itemize}
  \item $z$ is the optimal predictor for $e_1$ and $e_2$, respectively, and,
  \item The Bayesian optimal predictor for $e_1$ and $e_2$ together is also
    simply $z$.
  \end{itemize}
  We note that $w^*\circ\Phi^*$ gives the optimal predictor $z$,
  and also that $w^*$ is the optimal hypothesis for
  $\Phi^*(\widetilde{G}^{e_1})$ and $\Phi^*(\widetilde{G}^{e_2})$,
  respectively. Therefore $(w^*, \Phi^*)$ is an optimal solution to (IRM).
\end{proof}
\vskip 5pt
\noindent\textbf{From ``an'' optimal solution to ``any'' optimal solution.}
We have the following:
\begin{proposition}
  \label{prop:any-optimal}
  For $0$-$1$ loss, and any optimal solution $\overline{\Phi}, \overline{w}$ to
  (IRM), $\overline{w} \circ \overline{\Phi}$ must be $Z$ (i.e., the color).
\end{proposition}
\begin{proof}
  Consider any optimal solution $\overline{\Phi}$ and $\overline{w}$ to (IRM).
  It must satsify that its empirical loss across all environments must be
  upper bounded by that of $\Phi^*$ and $w^*$. That is,
  \begin{align*}
    R^{e_1}(\overline{w} \circ \overline{\Phi}) +
    R^{e_2}(\overline{w} \circ \overline{\Phi})
    \le R^{e_1}(w^* \circ \Phi^*) + R^{e_2}(w^* \circ \Phi^*).
  \end{align*}
  However $w^*\circ\Phi^*$ is the Bayesian optimal predictor $Z$.
  This means that $\overline{w} \circ \overline{\Phi}$ must also be $Z$.
  The proof is complete.
\end{proof}

Combining Propositions~\ref{prop:one-optimal} and~\ref{prop:any-optimal}
it shows that (IRM) cannot impose ECI and learn invariant correlations.

\section{Experimental Details for IRM under Representation Covariate Shift}
\label{sec:irm-misaligned-data}

There are two training environments $e_1, e_2$ and one testing environment $e_0$.
The data is generated with two control parameter $p$, $n$ as follows:
We first we assign a preliminary label $\tilde{y} = 0$ for digit $0-4$,
and $\tilde{y} = 1$ for digit $5-9$ for each data point in MNIST.
Then to create $e_1$, $e_2$, we randomly partition the $50000$ MNIST
training samples into two sets $S_1$ and $S_2$. In $e_1$, we sample $n$ points
with replacement from set $S_1$ to obtain data from 0-4 with probability
$\frac{p}{1+p}$ and data from 5-9 with probability
$\frac{1}{1+p}$; in $e_2$, we sample $n$ points
with replacement from set $S_2$ to obtain data from 0-4 with probability
$\frac{1}{1+p}$ and data from 5-9 with probability
$\frac{p}{1+p}$. Finally, we create final label (true label)
for data in all environments, $y$, by flipping $\tilde{y}$ with probability $0.25$.
Finally, we create the color variable for each sample $c$ by flipping $y$
with probability $q^e$, where $q^e = \begin{cases}
  0.2 & e = e_1 \\
  0.1 & e = e_2 \\
  0.9 & e = e_0 \\
\end{cases}$.

The result is given in Table \ref{tbl:irm-misalign}. We can observe that as $n$ increases, the train accuracy-test accuracy gap shrinks. As $p$ decreases, the training accuracy increases steadily. The test accuracy drops significantly in particular when $p$ goes from 0.6 to 0.3. The reason, we think, is that the IRM is no longer able to learn a useful representation from the two training environments with completely misaligned feature representations.

\begin{table}[htb!]
  \centering
  \begin{tabular}{cccc}
    \hline
    p                       & n                         & \begin{tabular}[c]{@{}c@{}}Training accuracy \\ (std dev.)\end{tabular} & \begin{tabular}[c]{@{}c@{}}Test accuracy \\ (std dev.)\end{tabular} \\ \hline
    1                       & 25000                     & 0.7141 (0.0095)                                                         & 0.6489 (0.0163)                                                     \\
    1                       & 50000                     & 0.6978 (0.0057)                                                         & 0.6955 (0.0079)                                                     \\
    1                       & 100000                    & 0.6995 (0.0057)                                                         & 0.6986 (0.0099)                                                     \\
    0.9                     & 25000                     & 0.7193 (0.0126)                                                         & 0.6578 (0.0158)                                                     \\
    0.9                     & 50000                     & 0.7059 (0.0056)                                                         & 0.6951 (0.0136)                                                     \\
    0.9                     & 100000                    & 0.7033 (0.0053)                                                         & 0.7087 (0.0092)                                                     \\
    0.8                     & 25000                     & 0.7152 (0.0072)                                                         & 0.6823 (0.0121)                                                     \\
    \multicolumn{1}{l}{0.8} & \multicolumn{1}{l}{50000} & \multicolumn{1}{l}{0.7107 (0.0053)}                                     & \multicolumn{1}{l}{0.6986 (0.0071)}                                 \\
    0.8                     & 100000                    & 0.7067 (0.0054)                                                         & 0.7025 (0.0092)                                                     \\
    0.7                     & 25000                     & 0.7347 (0.0122)                                                         & 0.6437 (0.0316)                                                     \\
    0.7                     & 50000                     & 0.7254 (0.0055)                                                         & 0.6724 (0.0124)                                                     \\
    0.7                     & 100000                    & 0.7198 (0.0032)                                                         & 0.6797 (0.0077)                                                     \\
    0.6                     & 25000                     & 0.7512 (0.0115)                                                         & 0.6126 (0.038)                                                      \\
    0.6                     & 50000                     & 0.7419 (0.0047)                                                         & 0.6332 (0.013)                                                      \\
    0.6                     & 100000                    & 0.7343 (0.0056)                                                         & 0.6388 (0.0161)                                                     \\
    0.5                     & 25000                     & 0.7767 (0.013)                                                          & 0.4915 (0.0583)                                                     \\
    0.5                     & 50000                     & 0.7551 (0.0067)                                                         & 0.5885 (0.0271)                                                     \\
    0.5                     & 100000                    & 0.7519 (0.0084)                                                         & 0.5981 (0.039)                                                      \\
    0.4                     & 25000                     & 0.7916 (0.0241)                                                         & 0.4089 (0.0991)                                                     \\
    0.4                     & 50000                     & 0.7828 (0.0152)                                                         & 0.4441 (0.0715)                                                     \\
    0.4                     & 100000                    & 0.7739 (0.0073)                                                         & 0.5053 (0.0392)                                                     \\
    0.3                     & 25000                     & 0.8356 (0.0065)                                                         & 0.2457 (0.0257)                                                     \\
    0.3                     & 50000                     & 0.8261 (0.0152)                                                         & 0.2756 (0.0497)                                                     \\
    0.3                     & 100000                    & 0.8277 (0.0078)                                                         & 0.2668 (0.0286)                                                     \\
    0.2                     & 25000                     & 0.8463 (0.0021)                                                         & 0.1879 (0.0095)                                                     \\
    0.2                     & 50000                     & 0.8444 (0.001)                                                          & 0.1801 (0.0067)                                                     \\
    0.2                     & 100000                    & 0.8425 (0.001)                                                          & 0.1853 (0.0054)                                                     \\
    0.1                     & 25000                     & 0.8465 (0.0017)                                                         & 0.1901 (0.0109)                                                     \\
    0.1                     & 50000                     & 0.8459 (0.0009)                                                         & 0.1717 (0.0127)                                                     \\
    0.1                     & 100000                    & 0.8455 (0.0007)                                                         & 0.1665 (0.0082)
  \end{tabular}
  \caption{Complete results of IRM under covariate shift. The covariate shift is created by manipulation of the data distribution described in the text in Section~\ref{sec:irm-misaligned-data}.}\label{tbl:irm-misalign}
\end{table}

\section{More Related Work}
\label{app:related}
\input{related}

\section{Relations between Our Bounds and Divergence-based Bounds}
\label{app:alter}

% Deriving Invariant Representation Learning Bound from Theorem~\ref{thm:env-generalization}.

\subsection{
  Review of the Divergence-based Bound for Single-Source Domain Adaptation}

The seminal work by~\cite{ben2010theory} considered the setting of single-source domain adaptation without representation learning, i.e., only considering $\mathcal{H}$ but not $\mathcal{F}$ or $\mathcal{G}$. It gives a bound on the risk in the target domain, based on the notion of $\mathcal{H}$-divergence. We review the divergence and the bound below.

By learning on the source, one cannot hope the learned hypothesis to generalize to arbitrary target. Therefore, some criterion is needed to measure how close the target is to the source. A na\"ive measurement is the $L_1$ distance.
However,~\cite{ben2010theory} pointed out the $L_1$ distance cannot be accurately estimated from finite samples of arbitrary distributions. Furthermore, it is a supremum over all measurable subsets while we are only interested in the risk of hypothesis from a class of finite complexity. They thus proposed to use the $\mathcal{H}$-divergence instead. The original bound is derived for the setting where the label $y \in [0, 1]$, the output of the hypothesis is in $\{0, 1\}$, and the loss is $\ell(y, y') = |y-y'|$. Here we gives a variant of the divergence and the original bound for general loss, which is convenient for the later discussion on comparison to our bounds.

\begin{definition}
  Denote the difference between the risks of two hypotheses $h, h'$ as
  \begin{align}
    \nu_e(h, h') = |R^{e}(h) - R^{e}(h')|.
  \end{align}
  The generalized $ \mathcal{H} \Delta  \mathcal{H}$-divergence between two distributions $e, e'$ is
  \begin{align}
    d_{ \mathcal{H} \Delta  \mathcal{H}}(e, e')
    & = 2 \sup_{h,h' \in  \mathcal{H} }  \left| \nu_{e}(h, h') - \nu_{e'}(h, h') \right|.
  \end{align}
\end{definition}

The generalized divergence upper bounds the change of the hypothesis risk difference due to distribution shifts. If it is small, then for any $h, h' \in  \mathcal{H}$ where $h$ has a smaller risk than $h'$ in $e$, we know that $h$ will also have a smaller (or not too larger) risk than $h'$ in $e'$. That is, if the divergence is small, then the ranking of the hypotheses w.r.t.\ the risk is roughly the same in both distributions. This \emph{rank-preserving} property makes sure that a good hypothesis learned in one domain will also be good for another.

\begin{theorem} \label{thm:divergence-bound-single-source}
  Suppose the loss is non-negative. For any $h \in  \mathcal{H}$,
  \begin{align}
    R^t(h)   \le &   \inf_{h^* \in  \mathcal{H}}  \left\{ R^t(h^*) + R^s(h^*) \right\}  +  R^s(h)  + d_{\mathcal{H} \Delta  \mathcal{H}}(s, t).
  \end{align}
\end{theorem}

\begin{proof}
  By definition of $  d_{\mathcal{H} \Delta  \mathcal{H}}( s, t)$ and non-negativity of the loss,
  \begin{align}
    & d_{\mathcal{H} \Delta  \mathcal{H}}( s, t)
    \\
    \ge &   \sup_{ h^* \in  \mathcal{H} }  \left\{ |\nu_{ t}(h, h^*) - \nu_{s}(h, h^*) | \right\}
    \\
    \ge &  \sup_{ h^* \in  \mathcal{H} }  \left\{ R^t(h) - R^t(h^*) - R^s(h)  - R^s(h^*) \right\}.
  \end{align}
  Rearranging the terms completes the proof.
\end{proof}

\subsection{Comparing Our Single-Source Bound to the Divergence-based Bound}

We can derive a bound by first applying the divergence-based bound Theorem~\ref{thm:divergence-bound-single-source} on the hypothesis class $\mathcal{H} = \{f^s_\phi \circ \phi, f^t_\phi \circ \phi\}$, and then bounding the divergence  with our notions $\KL_{\phi}^{s,t}, \KL_{\phi}^{t,s}, \delta_\phi^{s,t} $, and $\mu^{s,t}_\phi$.

\begin{proposition}\label{prop:connect}
  \begin{align*}
    R^t(f^s_\phi\circ\phi)
    \le 3 R^s(f^s_\phi\circ\phi)
    + \max\{\KL_{\phi}^{s,t}, \KL_{\phi}^{t,s} \}+ \delta_\phi^{s,t}
    + \mu^{s,t}_\phi.
  \end{align*}
\end{proposition}
\begin{proof}
  Recall $\KL_{\phi}^{t,s} = R^s(f^t_\phi\circ\phi) - R^s(f^s_\phi\circ\phi)$ and $\KL_{\phi}^{s,t} = R^t(f^s_\phi\circ\phi) - R^t(f^t_\phi\circ\phi)$.
  Applying the divergence-based bound Theorem~\ref{thm:divergence-bound-single-source} on the hypothesis class $\mathcal{H} = \{f^s_\phi \circ \phi, f^t_\phi \circ \phi\}$ gives:
  \begin{align*}
    R^t(f^s_\phi\circ\phi)
    \le R^s(f^s_\phi\circ\phi) + \min_{h \in \mathcal{H}} \{ R^s(h) + R^t(h) \} + |\KL_{\phi}^{t,s} - \KL_{\phi}^{s,t}|.
  \end{align*}
  If $\KL_{\phi}^{t,s} \ge \KL_{\phi}^{s,t}$, then
  \begin{align*}
    \min_{h \in \mathcal{H}} \{ R^s(h) + R^t(h) \} + |\KL_{\phi}^{t,s} - \KL_{\phi}^{s,t}|
    & \le \KL_{\phi}^{t,s} - \KL_{\phi}^{s,t} + R^s(f^s_\phi\circ\phi) + R^t(f^s_\phi\circ\phi)
    \\
    & \le \KL_{\phi}^{t,s} + R^s(f^s_\phi\circ\phi) + R^t(f^t_\phi\circ\phi).
  \end{align*}
  If $\KL_{\phi}^{t,s} \le \KL_{\phi}^{s,t}$, then
  \begin{align*}
    \min_{h \in \mathcal{H}} \{ R^s(h) + R^t(h) \} + |\KL_{\phi}^{t,s} - \KL_{\phi}^{s,t}|
    & \le -\KL_{\phi}^{t,s} + \KL_{\phi}^{s,t} + R^s(f^t_\phi\circ\phi) + R^t(f^t_\phi\circ\phi)
    \\
    & \le  \KL_{\phi}^{s,t} + R^s(f^s_\phi\circ\phi) + R^t(f^t_\phi\circ\phi).
  \end{align*}
  Then the statement follows from $R^t(f^t_\phi\circ\phi) - R^s(f^s_\phi\circ\phi) = \delta_\phi^{s,t}
  + \mu^{s,t}_\phi$.
\end{proof}

Our bound in Theorem~\ref{thm:exact-decomposition} is an equality and thus tighter than this, and the proof is simpler and more intuitive. The above proposition also shows that our bound gives a finer-grained analysis than the divergence-based bound Theorem~\ref{thm:divergence-bound-single-source}.

It is also instructive to apply Theorem~\ref{thm:divergence-bound-single-source} to explain Example~\ref{example:jsr}. If we apply it to $\mathcal{H} = \hyc \circ \rec$, then we can see that the first two terms
$\inf_{h^* \in  \mathcal{H}}  \left\{ R^t(h^*) + R^s(h^*) \right\} $ and  $R^s(h)$ can be small. However,  $d_{\mathcal{H} \Delta  \mathcal{H}}(s, t)$ will be large. Therefore, the bound can detect that the learned model may not generalize to the target domain, but it doesn't point out what leads to the problem, while our bound points out that the representation conditional label misalignment does.
Furthermore, the subtle issue in Example 1 arises when one applies Theorem~\ref{thm:divergence-bound-single-source} on the representation level instead of the input level.
More precisely, if we apply it on $\mathcal{H}_1 = \hyc \circ \{\phi_1\}$, we have
\begin{align}
  R^t(f \circ \phi_1)   \le &   \inf_{f^* \in  \hyc}  \left\{ R^t(f^* \circ \phi_1) + R^s(f^* \circ \phi_1) \right\}  +  R^s(f \circ \phi_1)  + d_{\mathcal{H}_1 \Delta  \mathcal{H}_1}(s, t).
\end{align}
Similarly, if we apply it on $\mathcal{H}_2 = \hyc \circ \{\phi_2\}$, we have
\begin{align}
  R^t(f \circ \phi_2)   \le &   \inf_{f^* \in  \hyc}  \left\{ R^t(f^* \circ \phi_2) + R^s(f^* \circ \phi_2) \right\}  +  R^s(f \circ \phi_2)  + d_{\mathcal{H}_2 \Delta  \mathcal{H}_2}(s, t).
\end{align}
The last two terms can be made small, but the generalization gap gets hidden in the first term. In particular, both $d_{\mathcal{H}_1 \Delta  \mathcal{H}_1}(s, t)$ and $d_{\mathcal{H}_2 \Delta  \mathcal{H}_2}(s, t)$ are 0, but
$d_{\mathcal{H} \Delta  \mathcal{H}}(s, t)$ can be large. Note that though $\mathcal{H} = \mathcal{H}_1 \cup \mathcal{H}_2$, $d_{\mathcal{H} \Delta  \mathcal{H}}(s, t)$ is much larger than the maximum of $d_{\mathcal{H}_1 \Delta  \mathcal{H}_1}(s, t)$ and $d_{\mathcal{H}_2 \Delta  \mathcal{H}_2}(s, t)$. The difference between $d_{\mathcal{H} \Delta  \mathcal{H}}(s, t)$ and $\max\{d_{\mathcal{H}_1 \Delta  \mathcal{H}_1}(s, t), d_{\mathcal{H}_2 \Delta  \mathcal{H}_2}(s, t)\}$ gets hidden in the first term, and is the root for the subtle issue in Example~\ref{example:jsr}.
In summary, using the bound on the input level is the correct way to apply it, which can detect there is an issue for generalization but still doesn't point out where the issue comes from.

\subsection{Generalizing the Divergence-based Bound to Multi-Source Domain Adaptation }

Here we show one can generalize the divergence-based bound for the case with a single source $s$ and target $t$ to the case with multiple sources $\etr$ and a target $\ete$.

Based on the divergence, we introduce the key notion for the analysis:
\begin{definition}
  The $ \mathcal{H}$-misalignment from $ \ete$ to $\etr$ is
  \begin{align}
    d_{ \mathcal{H}}( \ete; \etr)
    & = \inf_{e \in \etr} \left\{ \frac{1}{2} d_{ \mathcal{H} \Delta  \mathcal{H}}( \ete, e) \right \}
      = \inf_{e \in \etr}  \sup_{h,h' \in  \mathcal{H} }  \left | \nu_{ \ete}(h, h') - \nu_{e}(h, h') \right|.
  \end{align}
\end{definition}
The notion measures how aligned $ \ete$ is to $\etr$ w.r.t.\ risk ranking.
Intuitively, as long as there exists one $e \in \etr$ whose ranking of the hypotheses by their risks is similar to that of $ \ete$, then $ \ete$ is aligned to $\etr$.
To emphasize the difference from typical distribution distances, we use the term misalignment instead.

Then we can generalize Theorem~\ref{thm:divergence-bound-single-source} as follows.
\begin{theorem} \label{thm:env-generalization}
  Suppose the loss is non-negative. For  any $ \ete$ and any $h \in  \mathcal{H}$,
  \begin{align}
    R^{ \ete}(h)
    \le &
          \inf_{h^* \in  \mathcal{H}}  \left\{ R^{ \ete}(h^*) + \sup_{e \in \etr} R^{e}(h^*) \right\}  +  \sup_{e \in \etr} R^{e}(h)  + d_{ \mathcal{H}}( \ete; \etr).
  \end{align}
\end{theorem}

\begin{proof}
  By definition of $  d_{ \mathcal{H}}( \ete; \etr)$ and non-negativity of the loss,
  \begin{align}
    & d_{ \mathcal{H}}( \ete; \etr)
    \\
    \ge & \inf_{e \in \etr}  \sup_{ h^* \in  \mathcal{H} }  \left\{ |\nu_{ \ete}(h, h^*) - \nu_{e}(h, h^*) | \right\}
    \\
    \ge & \inf_{e \in \etr}  \sup_{ h^* \in  \mathcal{H} }  \left\{ R^{\ete}(h) - R^{\ete}(h^*) - R^{e}(h)  - R^{e}(h^*) \right\}.
  \end{align}
  Applying the max–min inequality and then rearranging the terms completes the proof.
\end{proof}

Similar to the single-source case, the bound in Theorem~\ref{thm:env-generalization} uses
$\inf_{h^* \in  \mathcal{H}}  \left\{ R^{ \ete}(h^*) \right.$ $\left.+ \sup_{e \in \etr} R^{e}(h^*) \right\}$ and $d_{ \mathcal{H}}( \ete; \etr)$. While our bound in Theorem~\ref{thm:multi-source} uses our notions of representation conditional label divergence, representation covariate shift, and prediction adaptation gap. The terms in Theorem~\ref{thm:env-generalization} can also be bounded using our notions using a similar argument as in Proposition~\ref{prop:connect}. Therefore, compared to the divergence-based bound, our bound provides a finer-grained analysis in the setting of representation learning.

%%% Local Variables:
%%% mode: latex
%%% TeX-master: t
%%% End:

%% file: upper-bound-cov-shift-via-fairness.tex
In this section we study a bound on the representation covariate shift
$\mu^{s,t}_\phi$ that has algorithmic implications. For ease of notations
we assume that there is no the singular part, but the argument here can
be easily extended to the situation with nontrivial singular part.

\noindent\textbf{Point Fairness}. We consider the following definition:

\begin{definition}[\textbf{Representation Source Fairness}]
  The representation source fairness $\rho^s_\phi$ is defined as
  $\rho^s_\phi \equiv \sup_{\gamma \in \Omega}$ $\big\{
  \Ent(Y^s|\,\gamma^s)\big\}$.
\end{definition}
Source fairness quantifies the intrinsic difficulties of $\phi$
in discriminating certain inputs (that is, even the Bayesian
optimal predictor over $\phi$ cannot discriminate the inputs mapped to $\gamma$
well). Intuitively, if $\phi$ is good at discriminating some inputs,
but very bad at some others, then $\phi$ is unfair to those inputs
(even though they may only occur with very small probability).

We note that, importantly, this quantity \emph{only} depends on the
\emph{source domain}, and so it is learnable using labeled source data.
Finally, observe that $\rho^s_\phi \le \log K$,
where the maximal is achieved when $Y^s|\, \gamma$ is a uniform
distribution over $[K]$. This leads to the following bound on
covariate shift.

\begin{theorem}
  \label{thm:covariate-shift-bound}
  \begin{align*}
    \mu^{s,t}_\phi
    \le \underbrace{
    \rho^s_\phi}_\text{repr. source fairness}
    \times \underbrace{
    \vphantom{\Big(\Big)}
    d_{\TV}(\Phi^s, \Phi^t)}_\text{
    repr. divergence}
  \end{align*}
\end{theorem}
\begin{proof}
  Note that $\kappa(\gamma) := H(Y^s|\, \gamma^s)/\rho^s_\phi$ is a function
  bounded by $1$, and $\mu_\phi^{s,t }$ is indeed
  $\rho^s_\phi \cdot \left(\int_\Omega \kappa(\gamma)\mu^t(d\gamma) -
    \int_\Omega \kappa(\gamma)\mu^s(d\gamma)\right)$,
  which is bounded by $\rho^s_\phi \cdot d_{\TV}(\Phi^s, \Phi^t)$
  where $d_{\TV}(\Phi^s, \Phi^t)$ is the total variation distance between
  $\Phi^s$ and $\Phi^t$.
\end{proof}

\noindent\textbf{Group Fairness}.
We can tighten the previous bound based on grouo fairness instead of point-wise
fairness. Let $\mathcal{B}$ be a partition of the space of the representation
$\phi$. Assume for simplicity $|\mathcal{B}|$ is finite.
\vskip 5pt

\begin{definition}[{\bf Group Representation Source Fairness}]
  The group (representation) source fairness $\rho^s_{\phi,\mathcal{B}}$
  with respect to $\mathcal{B}$ is defined as
  $\rho^s_{\phi,\mathcal{B}} := \sup_{ B \in \mathcal{B'} }\big\{
  \Ent(Y^s|\Phi^s \in B)\big\}$, where $\mathcal{B}'
  = \{B \in \mathcal{B}: \Pr[\Phi^s \in B] > 0\}$..
\end{definition}

\begin{definition}[{\bf Group Distance}]
  The group distance between two distributions $\mu$ and $\nu$ with respect to
  $\mathcal{B}$ is defined as
  $d_{\mathcal{B}}(\mu, \nu) =
  \frac{1}{2}\sum_{B \in \mathcal{B}} |\mu(B) - \nu(B)|$.
\end{definition}

\begin{theorem}
  \label{thm:covariate-shift-bound2}
  Suppose $\Phi^t$ is supported on $\Phi^s$, i.e., if $\Pr[\Phi^t \in B] > 0$
  for some set $B$, then $\Pr[\Phi^s \in B] > 0$. Then we have
  \begin{align*}
    \mu^{s,t}_\phi
    & \le \rho^s_{\phi,\mathcal{B}} \times d_{\mathcal{B}}(\Phi^s, \Phi^t) 
    \\
    &\le {\rho^s_\phi}
      \times {d_{\TV}(\Phi^s, \Phi^t)} \\
    & \le \log K \times d_{\TV}(\Phi^s, \Phi^t).
  \end{align*}
\end{theorem}
\begin{proof} 
  We have
  \begin{align*}
    \mu^{s,t}_\phi
    & = \int_\Omega \Ent(Y^s|\gamma^s)\mu^t(d\gamma)
      - \int_\Omega \Ent(Y^s|\gamma^s)\mu^s(d\gamma)\\
    & = \sum_{B \in \mathcal{B}' } \Pr[\Phi^t \in B ] \Ent(Y^s|\Phi^s \in B)
      - \sum_{B \in \mathcal{B}'} \Pr[\Phi^s \in B ] \Ent(Y^s|\Phi^s \in B) \\
    & =  \sum_{B \in \mathcal{B}' } \left(\Pr[\Phi^t \in B ]
      - \Pr[\Phi^s \in B ] \right)\Ent(Y^s|\Phi^s \in B)  \\
    & \le \sum_{B \in \mathcal{B}' } \max\left\{0, \Pr[\Phi^t \in B ]
      - \Pr[\Phi^s \in B ] \right\} \Ent(Y^s|\Phi^s \in B) \\
    & \le  \sup_{ B \in \mathcal{B}' }\Ent(Y^s|\Phi^s \in B)
      \times  \sum_{B \in \mathcal{B}' }
      \max\left\{0, \Pr[\Phi^t \in B ] - \Pr[\Phi^s \in B ] \right\} \\ 
    & = \sup_{ B \in \mathcal{B}' }\Ent(Y^s|\Phi^s \in B)
      \times \frac{1}{2}\sum_{B \in \mathcal{B}'} |\Pr[\Phi^t \in B ]
      - \Pr[\Phi^s \in B ] | \\
    & \le \rho^s_{\phi,\mathcal{B}} \times d_{\mathcal{B}}(\Phi^s, \Phi^t).
  \end{align*}
  Clearly, $\rho^s_{\phi,\mathcal{B}}  \le \rho^s_{\phi}$. Let $U(\mathcal{B})$
  be the family of sets that can be obtained by taking union of some sets in
  $\mathcal{B}$:
  \[
    U(\mathcal{B}) = \{U: U = \cup_{B \in \mathcal{A}} B, \mathcal{A}
    \subseteq \mathcal{B}\}.
  \]
  Then $d_{\mathcal{B}}(\mu, \nu) = \sup_{U \in U(\mathcal{B})}
  |\mu(U) - \nu(U)| \le d_{\TV}(\mu, \nu)$,
  where the last inequality follows from the definition of total variation
  distance. So the statement follows. 
\end{proof}

\noindent\textbf{Algorithmic Implications}. Note that both the point source
fairness and group source fairness depend only on the source domain,
and therefore one can hope to learn using labeled source data.
Our results thus show that by encouraging fairness, that is, the accuracy
being robust to change of source distributions, one can generalize better
in view of covariate shift in domain adaptation problems. In fact,
similar themes have been explored in some recent work,
such as~\cite{duchi2018learning} (but which is not at representation level).

%%% Local Variables:
%%% mode: latex
%%% TeX-master: t
%%% End:

%% file: related.tex
Representation learning has become a popular approach for various applications,
and learning invariant representations across multiple domains has been a
popular method for domain adaptation in recent years.  A classic approach for
analyzing domain adaption is based on
$\mathcal{H}$-divergence~\cite{kifer2004detecting,blitzer2008learning,ben2010theory}. 
That theoretical framework is the basis for a line of methods that uses
adversarial training with neural networks to learn representations that
are indistinguishable between source and target domain, in particular domain
adversarial neural network (DANN)~\cite{ajakan2014domain,ganin2016domain}
and related techniques~\cite{pei2018multi,zhao2018adversarial}.
Some other approach used different divergence notions, such as
MMD~\cite{long2014transfer,long2015learning},
Wasserstein distance~\cite{courty2017joint,shen2018wasserstein},
and R{\'e}nyi divergence~\cite{mansour2009multiple}.
Another line of research for domain adaptation is based on causal approaches that
typically assume shared generative distributions,
e.g.,~\cite{zhang2013domain,gong2016domain,azizzadenesheli2019regularized}.
This work instead focuses on discriminative representation learning and does not
make generative assumptions.

On the other hand, the $\mathcal{H}$-divergence bound is for general learning
rather than representation learning, and thus falls short in explaining some failure cases.
To this end, our bounds are finer-grained than the classic bounds for domain adaptation
based on $\mathcal{H}$-divergence, e.g., that by~\cite{ben2010theory}.
For single source, a similar bound as Theorem~\ref{thm:exact-decomposition}
can be derived from  the classic $\mathcal{H}$-divergence based bound,
by bounding the $\mathcal{H}$-divergence by the label divergence and covariate shift.
On the other hand, the bound in   Theorem~\ref{thm:exact-decomposition} is tighter
(it is an \emph{equality}!) and the analysis is more intuitive.
For multiple sources, we can also derive a multi-source $\mathcal{H}$-divergence
based bound. Our multi-source bound can also be viewed as decomposing
the $\mathcal{H}$-divergence into finer-grained quantities.
See Section~\ref{app:alter} in the appendix for the details.

Invariant Risk Minimization (IRM)~\cite{ABGLP19} proposed to learn
representations that result in the same optimal prediction across domains.
We noted that this corresponds to enforcing one factor in our risk decomposition,
which also reveals  conditions for success and suggests potential improvements to IRM.

%%% Local Variables:
%%% mode: latex
%%% TeX-master: t
%%% End: